\documentclass[11pt,letterpaper]{article}
\usepackage{graphicx}
\usepackage{caption}
\usepackage{epsfig,subcaption}
\usepackage{amssymb}
\usepackage{amsthm}
\usepackage{amsfonts}
\usepackage{amsmath}
\usepackage[margin=1.2in]{geometry}
\RequirePackage{natbib}
\RequirePackage[colorlinks,citecolor=blue,urlcolor=blue]{hyperref}

\usepackage{mkolar_definitions}

\newcommand{\ds}{{\Deltab_{\Sigmab}}}

\newcommand{\mt}{\mub_{T}}
\newcommand{\hmt}{\hat\mub_{T}}
\newcommand{\St}{\Sigmab_{TT}}
\newcommand{\hSt}{\Sbb_{TT}}     
\newcommand{\Sti}{\Sigmab_{TT}^{-1}}
\newcommand{\hSti}{\Sbb_{TT}^{-1}}
\newcommand{\sbt}{\sgn(\betab_T)}
\newcommand{\bt}{{\betab_T}}
\newcommand{\snbt}{\norm{\betab_T}_{\St}^2}
\newcommand{\hsnbt}{{\hmt'\hSti\hmt}}
\newcommand{\nbt}{\norm{\betab_T}_{\St}}

\newcommand{\hmp}[1]{\hat\mub_{#1}}
\newcommand{\hSip}[1]{\Sbb_{#1}^{-1}}
\newcommand{\hSp}[1]{\Sbb_{#1}}

\newcommand{\tmp}[1]{\mub_{#1}}
\newcommand{\tSip}[1]{\Sigmab_{#1}^{-1}}
\newcommand{\tSp}[1]{\Sigmab_{#1}}

\newcommand{\nr}{\frac{n_1n_2}{n(n-2)}}

\newcommand{\vsda}{\hat\vb^{\rm SDA}}

\def\supp{\mathop{\text{supp}\kern.2ex}}
\def\argmin{\mathop{\text{\rm arg\,min}}}
\def\argmax{\mathop{\text{\rm arg\,max}}}

\let\hat\widehat
\let\tilde\widetilde

\def\ds{\displaystyle}

\def\1{{(1)}}
\def\2{{(2)}}

\long\def\comment#1{}

\numberwithin{equation}{section}
\theoremstyle{plain}

\title{\huge Optimal Feature Selection in High-Dimensional
  Discriminant Analysis}
\date{}

\author{
Mladen Kolar\thanks{Machine Learning Department, Carnegie Mellon University, Pittsburgh, PA 15217, USA; e-mail: {\tt
mladenk@cs.cmu.edu}.} ~ and ~Han Liu\thanks{Department of Operations Research and Financial Engineering, Princeton University, Princeton, NJ 08544, USA; e-mail: {\tt
hanliu@princeton.edu} Research supported by  NSF Grant
IIS--1116730.}}

\bibpunct{(}{)}{,}{a}{,}{,}

\begin{document}

\maketitle

\begin{abstract}
 We consider the high-dimensional discriminant analysis problem. For
  this problem, different methods have been proposed and justified by
  establishing exact convergence rates for the classification risk, as
  well as the $\ell_2$ convergence results to the discriminative
  rule. However, sharp theoretical analysis for the variable selection
  performance of these procedures have not been established, even
  though model interpretation is of fundamental importance in
  scientific data analysis.  This paper bridges the gap by providing
  sharp sufficient conditions for consistent variable selection using
  the sparse discriminant analysis \citep{mai2012}. Through careful
  analysis, we establish rates of convergence that are significantly
  faster than the best known results and admit an optimal scaling of
  the sample size $n$, dimensionality $p$, and sparsity level $s$ in
  the high-dimensional setting.  Sufficient conditions are
  complemented by the necessary information theoretic limits on the
  variable selection problem in the context of high-dimensional
  discriminant analysis. Exploiting a numerical equivalence result,
  our method also establish the optimal results for the ROAD estimator
  \citep{fan2010road} and the sparse optimal scaling estimator
  \citep{clemmensen2011sparse}.  Furthermore, we analyze an exhaustive
  search procedure, whose performance serves as a benchmark, and show
  that it is variable selection consistent under weaker conditions.
  Extensive simulations demonstrating the sharpness of the bounds are
  also provided.
\end{abstract}

\textbf{Keywords}: {high-dimensional statistics}; {discriminant
  analysis}; {variable selection}; {optimal rates of convergence}

\section{Introduction}
\label{sec:introduction}

We consider the problem of binary classification with high-dimensional
features. More specifically, given $n$ data points, $\{(\xb_i,y_i),
i=1,...,n\}$, sampled from a joint distribution of $(\Xb,Y)\in
\mathbb{R}^{p}\times \{1,2\}$, we want to determine the class label
$y$ for a new data point $\xb \in \mathbb{R}^{p}$.

Let $p_1(\xb)$ and $p_2(\xb)$ be the density functions of $\Xb$ given
$Y=1$ (class 1) and $Y=2$ (class 2) respectively, and the prior
probabilities $\pi_1=\mathbb{P}(Y=1)$,
$\pi_2=\mathbb{P}(Y=2)$. Classical multivariate analysis theory shows
that the Bayes rule classifies a new data point $\xb$ to class $2$ if
and only if
\begin{equation}
\log\biggl( \frac{p_2(\xb)}{ p_1(\xb)}\biggr) + \log
\biggl(\frac{\pi_2}{\pi_1}\biggr) >0
\label{bayes_rule}.
\end{equation}
The Bayes rule usually serves as an oracle benchmark, since, in
practical data analysis, the class conditional densities $p_{2}(\xb)$
and $p_{1}(\xb)$ are unknown and need to be estimated from the data.

Throughout the paper, we assume that the class conditional
densities $p_{1}(\xb)$ and $p_{2}(\xb)$ are Gaussian. 
That is, we
assume that 
\begin{equation}
\label{eq:model}
\begin{aligned}
  \Xb | Y = 1 \sim \Ncal(\mub_1, \Sigmab)~~\text{and}~~
  \Xb | Y = 2 \sim \Ncal(\mub_2, \Sigmab).
\end{aligned}
\end{equation}
This assumption leads us to linear
discriminant analysis (LDA) and the Bayes rule in
\eqref{bayes_rule} becomes
\begin{equation}
  g(\xb; \mub_1,\mub_2,\Sigmab):=
  \left\{
\begin{array}{ll}
2 & \text{if }
\Bigl(\xb-(\mub_1+\mub_2)/2\Bigr)'\Sigmab^{-1}\mub
+ \log\left(\pi_{2}/\pi_{1} \right)> 0 \\
1 & \text{otherwise}
\end{array}
\right.
 \label{eq::LDAEQ}
\end{equation}
where $\mub = \mub_2 - \mub_1$.  Theoretical properties of the plug-in
rule $g(\xb; \hat\mub_1,\hat\mub_2,\hat\Sigmab)$, where
$(\hat\mub_1,\hat\mub_2,\hat\Sigmab)$ are sample estimates of
$(\mub_1,\mub_2,\Sigmab)$, have been well studied when the dimension
$p$ is low \citep{ande:1984}.

In high-dimensions, the standard plug-in rule works poorly and may
even fail completely.  For example, \citet{bickellevina:04} show that
the classical low dimensional normal-based linear discriminant
analysis is asymptotically equivalent to random guessing when the
dimension $p$ increases at a rate comparable to the sample size
$n$. To overcome this curse of dimensionality, it is common to impose
certain sparsity assumptions on the model and then estimate the
high-dimensional discriminant rule using plug-in estimators. The most
popular approach is to assume that both $\Sigmab$ and $\mub$ are
sparse. Under this assumption, \cite{shao2011lda} propose to use a
thresholding procedure to estimate $\Sigmab$ and $\mub$ and then plug
them into the Bayes rule.  In a more extreme case, \cite{tib2002,
  wang2007improved, fan2008high} assume that $\Sigmab = \Ib$ and
estimate $\mub$ using a shrinkage method. Another common approach is
to assume that $\Sigmab^{-1}$ and $\mub$ are sparse. Under this
assumption, \cite{scout} propose the scout method which estimates
$\Sigmab^{-1}$ using a shrunken estimator.  Though these plug-in
approaches are simple, they are not appropriate for conducting
variable selection in the discriminant analysis setting.  As has been
elaborated in \cite{cai2011lda} and \cite{mai2012}, for variable
selection in high-dimensional discriminant analysis, we need to
directly impose sparsity assumptions on the Bayes discriminant
direction $\betab = \Sigmab^{-1}\mub$ instead of separately on
$\Sigmab$ and $\mub$. In particular, it is assumed that $\betab =
(\betab_T', \zero')'$ for $T = \{1, \ldots, s\}$.  Their key
observation comes from the fact that the Fisher's discriminant rule
depends on $\Sigmab$ and $\mub$ only through the product
$\Sigmab^{-1}\mub$. Furthermore, in the high-dimensional setting, it
is scientifically meaningful that only a small set of variables are
relevant to classification, which is equivalent to the assumption that
$\betab$ is sparse. On a simple example of tumor classification,
\cite{mai2012} elaborate why it is scientifically more informative to
directly impose sparsity assumption on $\betab$ instead of on $\mub$
(For more details, see Section 2 of their paper).  In addition,
\cite{cai2011lda} point out that the sparsity assumption on $\betab$
is much weaker than imposing sparsity assumptions $\Sigmab^{-1}$ and
$\mub$ separately.  A number of authors have also studied
classification in this setting \citep{wu2009sparse, fan2010road,
  witten2011lda, clemmensen2011sparse, cai2011lda, mai2012}.
  
In this paper, we adopt the same assumption that $\betab$ is sparse
and focus on analyzing the SDA (Sparse Discriminant Analysis) proposed
by \citet{mai2012}. This method estimates the discriminant direction
$\betab$ (More precisely, they estimate a quantity that is
proportional to $\betab$.) and our focus will be on variable selection
consistency, that is, whether this method can recover the set $T$ with
high probability.  In a recent work, \citet{mai2012note} prove that
the SDA estimator is numerically equivalent to the ROAD estimator
\citep{fan2010road} and the sparse optimal scaling estimator
\citep{clemmensen2011sparse}. By exploiting this result, our
theoretical analysis provides a unified theoretical justification for
all these three methods.

\subsection{Main Results}
\label{sec::mainresult}

Let $n_1 = |\{i\ :\ y_i = 1\}|$ and $n_2 = n - n_1$.  The SDA
estimator is obtained by solving the following least squares
optimization problem
\begin{equation}
  \label{eq:opt_problem1}
  \min_{\vb \in \RR^{p}}\ 
  \frac{1}{2(n-2)}
  \sum_{i\in[n]}
  (z_i - \vb'(\xb_i - \bar\xb))^2 + \lambda \norm{\vb}_1,
\end{equation}
where $[n]$ denotes the set $\{1,\ldots,n\}$, $\bar \xb =
n^{-1}\sum_i\xb_i$ and the vector $\zb\in\RR^{n}$ encodes the class
labels as $z_i = n_2/n$ if $y_i=1$ and $z_i = -n_1/n$ if $y_i=2$. Here
$\lambda>0$ is a regularization parameter.

The SDA estimator in \eqref{eq:opt_problem1} uses an $\ell_1$-norm
penalty to estimate a sparse $\vb$ and avoid the curse of
dimensionality. \cite{mai2012} studied its variable selection property
under a different encoding scheme of the response $z_{i}$. However, as
we show later, different coding schemes do not affect the results.
When the regularization parameter $\lambda$ is set to zero, the SDA
estimator reduces to the classical Fisher's discriminant rule.

The main focus of the paper is to sharply characterize the variable
selection performance of the SDA estimator.  From a theoretical
perspective, unlike the high dimensional regression setting where
sharp theoretical results exist for prediction, estimation, and
variable selection consistency, most existing theories for
high-dimensional discriminant analysis are either on estimation
consistency or risk consistency, but not on variable selection
consistency \citep[see, for
example,][]{fan2010road,cai2011lda,shao2011lda}. \citet{mai2012}
provide a variable selection consistency result for the SDA estimator
in \eqref{eq:opt_problem1}. However, as we will show later, their
obtained scaling in terms of $(n,p,s)$ is not optimal. Though some
theoretical analysis of the $\ell_{1}$-norm penalized M-estimators
exists (see \cite{Wain:09a, negahban2010unified}), these techniques
are not applicable to analyze the estimator given in
\eqref{eq:opt_problem1}.  In high-dimensional discriminant analysis
the underlying statistical model is fundamentally different from that
of the regression analysis. At a high level, to establish variable
selection consistency of the SDA estimator, we characterize the
Karush-Kuhn-Tucker (KKT) conditions for the optimization problem in
\eqref{eq:opt_problem1}.  Unlike the $\ell_1$-norm penalized least
squares regression, which directly estimates the regression
coefficients, the solution to \eqref{eq:opt_problem1} is a quantity
that is only proportional to the Bayes rule's direction
$\betab=\Sti\mt$.  To analyze such scaled estimators, we need to
resort to different techniques and utilize sophisticated multivariate
analysis results to characterize the sampling distributions of the
estimated quantities. More specifically, we provide sufficient
conditions under which the SDA estimator is variable selection
consistent with a significantly improved scaling compared to that
obtained by \cite{mai2012}. In addition, we complement these
sufficient conditions with information theoretic limitations on
recovery of the feature set $T$. In particular, we provide lower
bounds on the sample size and the signal level needed to recover the
set of relevant variables by any procedure. We identify the family of
problems for which the estimator \eqref{eq:opt_problem1} is variable
selection optimal.  To provide more insights into the problem, we
analyze an exhaustive search procedure, which requires weaker
conditions to consistently select relevant variables. This estimator,
however, is not practical and serves only as a benchmark.  The
obtained variable selection consistency result also enables us to
establish risk consistency for the SDA estimator.  In addition,
\cite{mai2012note} show that the SDA estimator is numerically
equivalent to the ROAD estimator proposed by \cite{wu2009sparse,
  fan2010road} and the sparse optimal scaling estimator proposed by
\cite{clemmensen2011sparse}.  Therefore, the results provided in this
paper also apply to those estimators.  Some of the main results of
this paper are summarized below.

Let $\vsda$ denote the minimizer of \eqref{eq:opt_problem1}. We show that
if the sample size
\begin{equation}\label{eq::result1}
    n \geq C \rbr{\max_{a\in N}\Sigma_{a|T}}
           \Lambda_{\min}^{-1}(\St)
           s\log\rbr{(p-s)\log(n)},
\end{equation}
where $C$ is a fixed constant which does not scale with $n, p$ and
$s$, $\sigma_{a|T} = \sigma_{aa} - \Sigmab_{aT}\Sti\Sigmab_{Ta}$, and
$\Lambda_{\min}(\Sigmab)$ denotes the minimum eigenvalue of $\Sigmab$,
then the estimated vector $\vsda$ has the same sparsity pattern as the
true $\betab$, thus establishing variable selection consistency (or
sparsistency) for the SDA estimator.  This is the first result that
proves that consistent variable selection in the discriminant analysis
can be done under a similar theoretical scaling as variable selection
in the regression setting (in terms of $n, p$ and $s$).  To
prove~\eqref{eq::result1}, we impose conditions that $ \min_{j \in T}
|\beta_j| $ is not too small and
$\norm{\Sigmab_{NT}\Sigmab_{TT}^{-1}\sbt}_{\infty} \leq 1 - \alpha$
with $\alpha\in(0,1)$, where $N = [p] \bks T$.  The latter one is the
irrepresentable condition, which is commonly used in the $\ell_1$-norm
penalized least squares regression problem
\citep{Zou:2006,Meinshausen:06,ZY07,Wain:09a}.  Let $\beta_{\min}$ be
the magnitude of the smallest absolute value of the non-zero component
of $\betab$. Our analysis of information theoretic limitations reveals
that, whenever $n < \small{C_{1}\beta_{\min}^{-2}\log(p-s)}$, no
procedure can reliably recover the set $T$. In particular, under
certain regimes, we establish that the SDA estimator is optimal for
the purpose of variable selection. The analysis of the exhaustive
search decoder reveals a similar result. However, the exhaustive
search decoder does not need the irrepresentable condition to be
satisfied by the covariance matrix. Thorough numerical simulations are
provided to demonstrate the sharpness of our theoretical results.

In a preliminary work, \citet{kolar13road} present some variable
selection consistency results related to the ROAD estimator under the
assumption that $\pi_{1}=\pi_{2}=1/2$.  However, it is hard to
directly compare their analysis with that of \citet{mai2012} to
understand why an improved scaling is achievable, since the ROAD
estimator is the solution to a constrained optimization while the SDA
estimator is the solution to an unconstrained optimization.  This
paper analyzes the SDA estimator and is directly comparable with the
result of \citet{mai2012}. As we will discuss later, our analysis
attains better scaling due to a more careful characterization of the
sampling distributions of several scaled statistics. In contrast, the
analysis in \citet{mai2012} hinges on the sup-norm control of the
deviation of the sample mean and covariance to their population
quantities, which is not sufficient to obtain the optimal rate.  Using
the numerical equivalence between the SDA and the ROAD estimator, the
theoretical results of this paper also apply on the ROAD
estimator. In addition, we also study an exhaustive search decoder and
information theoretic limits on the variable selection in
high-dimensional discriminant analysis. Furthermore, we provide
discussions on risk consistency and approximate sparsity, which shed
light on future investigations.

The rest of this paper is organized as follows. In the rest of this
section, we introduce some more notation.  In $\S$\ref{sec:sda}, we
study sparsistency of the SDA estimator.  An information theoretic
lower bound is given in $\S$\ref{sec:lower-bound}. We characterize the
behavior of the exhaustive search procedure in
$\S$\ref{sec:exhaustive_search}. Consequences of our results are
discussed in more details in $\S$\ref{sec:consequence}.  Numerical
simulations that illustrate our theoretical findings are given in
$\S$\ref{sec:simulation}.  We conclude the paper with a discussion and
some results on the risk consistency and approximate sparsity in
$\S$\ref{sec:discussion}. Technical results and proofs are deferred to
the appendix.

\subsection{Notation}

We denote $[n]$ to be the set $\{1,\ldots,n\}$.  Let $T
\subseteq [p]$ be an index set, we denote $\betab_T$ to be the
subvector containing the entries of the vector $\betab$ indexed by the
set $T$, and $\Xb_T$ denotes the submatrix containing the columns of
$\Xb$ indexed by $T$. Similarly, we denote $\Ab_{TT}$ to be the
submatrix of $\Ab$ with rows and columns indexed by $T$. For a vector
$\ab \in \RR^n$, we denote ${\rm supp}(\ab) = \{j\ :\ a_j \neq 0\}$ to
be the support set. We also use $\norm{\ab}_q$, $q \in [1,\infty)$, to
be the $\ell_q$-norm defined as $\norm{\ab}_q = (\sum_{i\in[n]}
|a_i|^q)^{1/q}$ with the usual extensions for $q \in \{0,\infty\}$,
that is, $\norm{\ab}_0 = |{\rm supp}(\ab)|$ and $\norm{\ab}_\infty =
\max_{i\in[n]}|a_i|$.  For a matrix $\Ab \in \RR^{n \times p}$, we
denote $\opnorm{\Ab}{\infty} = \max_{i\in[n]}\sum_{j\in[p]}|a_{ij}|$
the $\ell_\infty$ operator norm.  For a symmetric positive definite
matrix $\Ab \in \RR^{p \times p}$ we denote $\Lambda_{\min}(\Ab)$ and
$\Lambda_{\max}(\Ab)$ to be the smallest and largest eigenvalues,
respectively. We also represent the quadratic form $\norm{\ab}_{\Ab}^2
= \ab'\Ab\ab$ for a symmetric positive definite matrix $\Ab$.  We
denote $\Ib_{n}$ to be the $n\times n$ identity matrix and $\one_n$ to
be the $n \times 1$ vector with all components equal to $1$. For two
sequences $\{a_n\}$ and $\{b_n\}$, we use $a_n = \Ocal(b_n)$ to denote
that $a_n < Cb_n$ for some finite positive constant $C$.  We also
denote $a_{n} = \Ocal(b_{n})$ to be $b_{n}\gtrsim a_{n}$. If $a_{n} =
\Ocal(b_{n})$ and $b_{n} = \Ocal(a_{n})$, we denote it to be
$a_{n}\asymp b_{n}$. The notation $a_n = o(b_n)$ is used to denote
that $a_nb_n^{-1}\rightarrow 0$.

\section{Sparsistency of the SDA Estimator}
\label{sec:sda}

In this section, we provide sharp sparsistency analysis for the SDA
estimator defined in \eqref{eq:opt_problem1}.  Our analysis decomposes
into two parts: (i) We first analyze the population version of the SDA
estimator in which we assume that $\Sigmab$, $\mub_{1}$, and
$\mub_{2}$ are known.  The solution to the population problem provides
us insights on the variable selection problem and allows us to write
down sufficient conditions for consistent variable selection. (ii) We
then extend the analysis from the population problem to the sample
version of the problem in \eqref{eq:opt_problem1}.  For this, we need
to replace $\Sigmab$, $\mub_{1}$, and $\mub_{2}$ by their
corresponding sample estimates $\hat{\Sigmab}$, $\hat{\mub}_{1}$, and
$\hat{\mub}_{2}$.  The statement of the main result is provided in
$\S$\ref{sec:sparsistency_sda_sample} with an outline of the proof in
$\S$\ref{sec:sparsistency_sda_sample:proof}.

\subsection{Population Version Analysis of the SDA Estimator}
\label{sec:sda_population}

We first lay out conditions that characterize the solution to the
population version of the SDA optimization problem. 

Let $\Xb_1 \in \RR^{n_1 \times p}$ be the matrix with rows containing
data points from the first class and similarly define $\Xb_2 \in
\RR^{n_2\times p}$ to be the matrix with rows containing data points
from the second class.  We denote $ \Hb_1 = \Ib_{n_1} -
n_1^{-1}\one_{n_1}\one_{n_1}'$ and $ \Hb_2 = \Ib_{n_2} -
n_2^{-1}\one_{n_2}\one_{n_2}'$ to be the centering matrices.  We
define the following quantities
\begin{gather*}
  \hat \mub_1 = n_1^{-1}\sum_{i:y_i=1}\xb_i = n_1^{-1}\Xb_1'\one_{n_1},
\quad
  \hat \mub_2 = n_2^{-1}\sum_{i:y_i=2}\xb_i = n_2^{-1}\Xb_2'\one_{n_2}, \\ 
  \hat \mub = \hat\mub_2 - \hat\mub_1, \\ 
\Sbb_1 = (n_1-1)^{-1}\Xb_1'\Hb_1\Xb_1, 
\quad
\Sbb_2 = (n_2-1)^{-1}\Xb_2'\Hb_2\Xb_2, \\
\Sbb = (n - 2)^{-1}( (n_1-1)\Sbb_1 + (n_2 - 1)\Sbb_2).
\end{gather*}
With this notation, observe that the optimization problem in
\eqref{eq:opt_problem1} can be rewritten as
\[
  \min_{\vb \in \RR^{p}}\ 
  \frac{1}{2}
\vb'
\rbr{
\Sbb + 
\frac{n_1n_2}{n(n-2)}\hat\mub\hat\mub'
}
\vb 
-
\frac{n_1n_2}{n(n-2)}\vb'\hat\mub + \lambda\norm{\vb}_1,
\]
where we have dropped terms that do not depend on $\vb$.  Therefore,
we define the population version of the SDA optimization problem as
\begin{equation}
  \label{eq:population:optimization}
  \min_{\wb}\ 
  \frac{1}{2}
\wb'
\rbr{
\Sigmab + 
\pi_1\pi_2\mub\mub'
}
\wb 
-
\pi_1\pi_2\wb'\mub + \lambda\norm{\wb}_1,
\end{equation}
Let $\hat\wb$ be the solution of \eqref{eq:population:optimization}.
We are aiming to characterize conditions under which the solution
$\hat \wb$ recovers the sparsity pattern of $\betab=\Sigmab^{-1}\mub$.
Recall that $T = \supp(\betab) = \{1,\ldots, s\}$ denotes the true
support set and $N = [p] \bks T$, under the sparsity assumption, we
have
\begin{equation}
\label{eq:beta_t_mu_n}
  \betab_T = \Sigmab_{TT}^{-1}\mub_{T}
\quad\text{and}\quad
  \mub_{N} = \Sigmab_{NT}\Sigmab_{TT}^{-1}\mub_{T}.
\end{equation}
We define $\beta_{\min}$ as 
\begin{equation}
  \label{eq:beta_min:population}
  \beta_{\min} = \min_{a \in T} |\beta_a|.
\end{equation} 
The following theorem characterizes the solution to the population
version of the SDA optimization problem in
\eqref{eq:population:optimization}.
\begin{theorem}
  \label{thm:sda_population}
  Let $\alpha \in (0, 1]$ be a constant and $\hat\wb$ be the solution to the
  problem in \eqref{eq:population:optimization}. Under the assumptions that 
    \begin{align}
 &  \norm{\Sigmab_{NT}\Sigmab_{TT}^{-1}\sbt}_{\infty} \leq 1 - \alpha,
      \label{eq:assum:irrepresentable}
\\
&  \pi_1\pi_2\frac{1+\lambda\norm{\bt}_1}{1+\pi_1\pi_2\snbt}
   \beta_{\min} > 
    \lambda \norm{\Sti\sbt}_\infty,  \label{eq:assum:beta_min}
   \end{align}  
  we have $\hat\wb = (\hat\wb_T', \zero')$ with 
  \begin{equation} \label{wbT}
    \hat\wb_T = \pi_1\pi_2 
    \frac{1+\lambda\norm{\bt}_1}
         {1+\pi_1\pi_2\snbt}
    \bt - 
    \lambda\Sigmab_{TT}^{-1}\sbt.
  \end{equation}
  Furthermore, we have $\sgn(\hat\wb_T) = \sgn(\betab_T)$.
\end{theorem}
Equations \eqref{eq:assum:irrepresentable} and
\eqref{eq:assum:beta_min} provide sufficient conditions under which
the solution to \eqref{eq:population:optimization} recovers the true
support.  The condition in \eqref{eq:assum:irrepresentable} takes the
same form as the irrepresentable condition commonly used in the
$\ell_{1}$-penalized least squares regression problem
\citep{Zou:2006,Meinshausen:06,ZY07,Wain:09a}.
Equation~\eqref{eq:assum:beta_min} specifies that the smallest
component of $\betab_T$ should not be too small compared to the
regularization parameter $\lambda$.  In particular, let
$\lambda=\lambda_0/(1+\pi_1\pi_2\snbt)$ for some $\lambda_0$. Then
\eqref{eq:assum:beta_min} suggests that $\hat\wb_T$ recovers
the true support of $\betab$ as long as $\beta_{\min} \geq
\lambda_0\norm{\Sti\sbt}_\infty$.  Equation~\eqref{wbT} provides an
explicit form for the solution $\hat\wb$, from which we see that the
SDA optimization procedure estimates a scaled version of the optimal
discriminant direction. Whenever $\lambda\neq0$, $\hat \wb$ is a
biased estimator.  However, such estimation bias does not affect the
recovery of the support set $T$ of $\betab$ when $\lambda$ is small
enough.

We present the proof of Theorem~\ref{thm:sda_population}, as the
analysis of the sample version of the SDA estimator will follow the
same lines. We start with the Karush-Kuhn-Tucker (KKT) conditions for
the optimization problem in \eqref{eq:population:optimization}:
\begin{align}
\label{eq:KKT}
  \rbr{\Sigmab + \pi_1\pi_2\mub\mub'}\hat\wb -\pi_1\pi_2\mub+ \lambda \hat\zb & = \zero
\end{align}
where  $\hat\zb \in \partial \norm{\hat\wb}_1$ is an element
of the subdifferential of $\|\cdot\|_{1}$.

Let $\hat\wb_T$ be defined in \eqref{wbT}. We need to show that there
exists a $\hat\zb$ such that the vector $\hat \wb = (\hat \wb_T',
\zero')'$, paired with $\hat\zb$, satisfies the KKT conditions and
$\sgn(\hat{\wb}_T) = \sgn(\bt)$.

The explicit form of $\hat\wb_T$ is obtained as the solution to an
oracle optimization problem, specified in
\eqref{eq:oracle:population:optimization}, where the solution is
forced to be non-zero only on the set $T$. Under the assumptions of
Theorem~\ref{thm:sda_population}, the solution $\hat\wb_T$ to the
oracle optimization problem satisfies $\sgn(\hat{\wb}_T) =
\sgn(\bt)$. We complete the proof by showing that the vector
$(\hat\wb_T', \zero')'$ satisfies the KKT conditions for the full
optimization procedure.

We define the oracle optimization problem to be 
\begin{equation}
  \label{eq:oracle:population:optimization}
  \min_{\wb_T}\ 
  \frac{1}{2}
\wb_T'
\rbr{
\St + 
\pi_1\pi_2\mt\mt'
}
\wb_T
-
\pi_1\pi_2\wb_T'\mt + \lambda\wb_T'\sbt.
\end{equation}
The solution $\tilde \wb_T$ to the oracle optimization problem
\eqref{eq:oracle:population:optimization} satisfies $\tilde\wb_T =
\hat \wb_T$ where $\hat\wb_T$ is given in \eqref{wbT}. It is
immediately clear that under the conditions of
Theorem~\ref{thm:sda_population}, $\sgn\rbr{\tilde\wb_T} = \sbt$.

The next lemma shows that the vector $(\hat\wb_T', \zero')'$ is the
solution to the optimization problem in
\eqref{eq:population:optimization} under the assumptions of
Theorem~\ref{thm:sda_population}.
\begin{lemma}
  \label{lem:global_solution}
  Under the conditions of Theorem~\ref{thm:sda_population}, we have
  that $\hat\wb = (\hat\wb_T', \zero')$ is the solution to the problem
  in \eqref{eq:population:optimization}, where $\hat\wb_T$ is defined
  in~\eqref{wbT}.  
\end{lemma}

This completes the proof of Theorem~\ref{thm:sda_population}.

The next theorem shows that the irrepresentable condition in
\eqref{eq:assum:irrepresentable} is almost necessary for sign
consistency, even if the population quantities $\Sigmab$ and $\mub$
are known.
\begin{theorem}
  \label{thm:necessary} 
  Let $\hat\wb$ be the solution to the problem in
  \eqref{eq:population:optimization}.  If we have $\sgn(\hat\wb_T) =
  \sbt$, Then, there must be
    \begin{equation}
    \label{eq:assum:irrepresentablenecessary}
    \norm{\Sigmab_{NT}\Sigmab_{TT}^{-1}\sbt}_{\infty} \leq 1.
  \end{equation}  
\end{theorem}
The proof of this theorem follows similar argument as in the
regression settings in \cite{ZY07, Zou:2006}.

\subsection{Sample Version Analysis of the SDA Estimator}
\label{sec:sparsistency_sda_sample}

In this section, we analyze the variable selection performance of the
sample version of the SDA estimator $\hat\vb = \vsda$ defined in
\eqref{eq:opt_problem1}. In particular, we will establish sufficient
conditions under which $\hat\vb$ correctly recovers the support set of
$\betab$ (i.e., we will derive conditions under which $\hat\vb =
(\hat\vb_T', \zero')'$ and $\sgn(\hat{\vb}_T) = \sbt$).  The proof
construction follows the same line of reasoning as the population
version analysis. However, proving analogous results in the sample
version of the problem is much more challenging and requires careful
analysis of the sampling distribution of the scaled functionals of
Gaussian random vectors.

The following theorem is the main result that characterizes the
variable selection consistency of the SDA estimator.
\begin{theorem}
  \label{thm:main:sda}
  We assume that the  condition in \eqref{eq:assum:irrepresentable} holds.
  Let the penalty parameter be $\lambda = \rbr{1+\pi_1\pi_2\snbt}^{-1}\lambda_0$
  with
  \begin{equation}
    \label{eq:lambda0:thm}
    \lambda_0 = K_{\lambda_0}
       \sqrt{\pi_1\pi_2
         \rbr{\max_{a\in N}\sigma_{a|T}} \rbr{1\vee\snbt}
         \frac{\log\rbr{(p-s)\log(n)}}{n}
       }       
  \end{equation}
  where $ K_{\lambda_0}$ is a sufficiently large constant.
  Suppose that $\beta_{\min} = \min_{a \in T} |\beta_a|$ satisfies 
  \begin{equation}
    \label{eq:beta_min:lemma:1}
    \begin{aligned}
    \beta_{\min} &\geq K_{\beta} 
     \Bigl(
      \sqrt{\rbr{\max_{a\in T}
        \rbr{\Sti}_{aa}}\rbr{1\vee\snbt}
        \frac{\log(s\log(n))}{n}
      }   \\
      & \qquad\qquad\qquad
      \bigvee
      \lambda_0\norm{\Sti\sbt}_\infty
     \Bigl) 
   \end{aligned}
 \end{equation}
  for a sufficiently large constant $K_{\beta}$. If 
  \begin{equation}
    \label{eq:sample_size:thm}
    n \geq K \pi_1\pi_2 \rbr{\max_{a\in N}\sigma_{a|T}}
           \Lambda_{\min}^{-1}(\St)
           s\log\rbr{(p-s)\log(n)}
  \end{equation}
  for some constant $K$, then $\hat\wb = (\hat\wb_T',\zero')'$ is the
  solution to the optimization problem in
  \eqref{eq:opt_problem1}, where
\begin{equation}
  \label{eq:solution:hat_wt}
\begin{aligned}  
  \hat \vb_T
  & = 
    \frac{n_1n_2}{n(n-2)}
    \frac{1 + \lambda\norm{\hat\bt}_1}
      {1+\frac{n_1n_2}{n(n-2)}\norm{\hat\bt}_{\hSt}^2}
    \hat\bt -
    \lambda\hSti\sgn(\hat\bt)\ \text{and} \\
  \hat\betab_T & = \hSti\hmt,
\end{aligned}
\end{equation}
with probability at least $1-\Ocal\rbr{\log^{-1}(n)}$. Furthermore,
$\sgn(\hat \vb_T) = \sbt$.
\end{theorem}
Theorem~\ref{thm:main:sda} is a sample version of
Theorem~\ref{thm:sda_population} given in the previous
section. Compared to the population version result, in addition to the
irrepresentable condition and a lower bound on $\beta_{\min}$, we also
need the sample size $n$ to be large enough for the SDA procedure to
recover the true support set $T$ with high probability.

At the first sight, the conditions of the theorem look complicated. To
highlight the main result, we consider a case where $0 < \underbar c
\leq \Lambda_{\min}(\St)$ and $\rbr{\snbt \bigvee
  \norm{\Sti\sbt}_\infty} \leq \bar C < \infty$ for some constants
$\underbar c, \bar C$. In this case, it is sufficient that the sample
size scales as $n \asymp s \log(p-s)$ and $\beta_{\min} \gtrsim
s^{-1/2}$.  This scaling is of the same order as for the Lasso
procedure, where $n \gtrsim s\log(p-s)$ is needed for correct recovery
of the relevant variables under the same assumptions \citep[see
Theorem 3 in][]{Wain:09a}. In $\S$\ref{sec:consequence}, we provide
more detailed explanation of this theorem and complement it with the
necessary conditions given by the information theoretic limits.

Variable selection consistency of the SDA estimator was studied by
\citet{mai2012}. Let $\Cb = \Var(\Xb)$ denote the marginal covariance
matrix. Under the assumption that
$\opnorm{\Cb_{NT}\Cb_{TT}^{-1}}{\infty}$,
$\opnorm{\Cb_{TT}^{-1}}{\infty}$ and $\norm{\mub}_\infty$ are bounded,
\citet{mai2012} show that the following conditions
\[
i)\ \lim_{n\rightarrow \infty}\frac{s^{2}\log p}{n}=0,
\quad \text{and} \quad
ii)\ \beta_{\min} \gg \sqrt{\frac{s^2\log(ps)}{n}}
\]
are sufficient for consistent support recovery of $\betab$. This is
suboptimal compared to our results. Inspection of the proof given in
\citet{mai2012} reveals that their result hinges on uniform control of
the elementwise deviation of $\hat \Cb$ from $\Cb$ and $\hat \mub$
from $\mub$. These uniform deviation controls are too rough to
establish sharp results given in Theorem~\ref{thm:main:sda}. In our
proofs, we use more sophisticated multivariate analysis tools to
control the deviation of $\hat \bt$ from $\bt$, that is, we focus on
analyzing the quantity $\hSti\hmt$ but instead of studying $\hSt$ and
$\hmt$ separately.

The optimization problem in \eqref{eq:opt_problem1} uses a particular
scheme to encode class labels in the vector $\zb$, though other
choices are possible as well. For example, suppose that we choose $z_i
= z^{(1)}$ if $y_i=1$ and $z_i = z^{(2)}$ if $y_i=2$, with $z^{(1)}$
and $z^{(2)}$ such that $n_1z^{(1)} + n_2z^{(2)} = 0$. The optimality
conditions for the vector $\tilde\vb = (\tilde\vb_{\tilde T}',
\zero')'$ to be a solution to \eqref{eq:opt_problem1} with the
alternative coding are
  \begin{align}
  \label{eq:KKT:1:mod}
\rbr{
\Sbb_{\tilde T \tilde T} + 
\frac{n_1n_2}{n(n-2)}\hat\mub_{\tilde T}\hat\mub_{\tilde T}'
}
\tilde\wb_{\tilde T} = \frac{n_1z^{(1)}}{n-2}\hat\mub_{\tilde T} - \tilde\lambda\sgn(\tilde\wb_{\tilde T})
\\
\bignorm{
\rbr{
\Sbb_{\tilde N \tilde T} + \frac{n_1n_2}{n(n-2)}\hat\mub_{\tilde N}\hat\mub_{\tilde T}'
}
\tilde\wb_{\tilde T} - \frac{n_1z^{(1)}}{n-2}\hat\mub_{\tilde N}
}_{\infty} \leq \tilde\lambda.
  \label{eq:KKT:2:mod}
\end{align}
Now, choosing $\tilde\lambda = \frac{z^{(1)}n}{n_2} \lambda$, we
obtain that  $\tilde \wb_{\tilde T}$, which satisfies
\eqref{eq:KKT:1:mod} and \eqref{eq:KKT:2:mod}, is proportional to
$\hat\wb_T$ with $\tilde T = T$.  Therefore, the choice of different 
coding schemes of the response variable $z_{i}$ does not effect the result.

The proof of Theorem~\ref{thm:main:sda} is outlined in the next
subsection.

\subsection{Proof of Sparsistency of the SDA Estimator}
\label{sec:sparsistency_sda_sample:proof}

The proof of Theorem~\ref{thm:main:sda} follows the same strategy as
the proof of Theorem~\ref{thm:sda_population}.  More specifically, we
only need to show that there exists a subdifferential of
$\|\cdot\|_{1}$ such that the solution $\hat\vb$ to the optimization
problem in \eqref{eq:opt_problem1} satisfies the sample version KKT
condition with high probability.  For this, we proceed in two
steps. In the first step, we assume that the true support set $T$ is
known and solve an oracle optimization problem to get $\tilde\vb_T$
which exploits the knowledge of $T$.  In the second step, we show that
there exists a dual variable from the subdifferential of
$\|\cdot\|_{1}$ such that the vector $(\tilde\vb_T', \zero')'$, paired
with $(\tilde\vb_T', \zero')'$, satisfies the KKT conditions for the
original optimization problem given in \eqref{eq:opt_problem1}.  This
proves that $\hat\vb = (\tilde\vb_T', \zero')'$ is a global minimizer
of the problem in \eqref{eq:opt_problem1}. Finally, we show that
$\hat\vb$ is a unique solution to the optimization problem in
\eqref{eq:opt_problem1} with high probability.

Let $\hat T = \supp(\hat \vb)$ be the support of a solution $\hat \vb$
to the optimization problem in \eqref{eq:opt_problem1} and $\hat N =
[p]\bks\hat T$. Any solution to \eqref{eq:opt_problem1} needs to
satisfy the following Karush-Kuhn-Tucker (KKT) conditions
\begin{align}
  \label{eq:KKT:1}
\rbr{
\Sbb_{\hat T \hat T} + \frac{n_1n_2}{n(n-2)}\hat\mub_{\hat T}\hat\mub_{ \hat T}'
}
\hat\vb_{\hat T} = \frac{n_1n_2}{n(n-2)}\hat\mub_{\hat T} - \lambda\sgn(\hat\vb_{\hat T}),
\\
\bignorm{
\rbr{
\Sbb_{\hat N \hat T} + \frac{n_1n_2}{n(n-2)}\hat\mub_{\hat N}\hat\mub_{ \hat T}'
}
\hat\vb_{\hat T} - \frac{n_1n_2}{n(n-2)}\hat\mub_{\hat N}
}_{\infty} \leq \lambda.
  \label{eq:KKT:2}
\end{align}
We construct a solution $\hat\vb = (\hat \vb_T', \zero')'$ to
\eqref{eq:opt_problem1} and show that it is unique with
high probability.

First, we consider the following oracle optimization problem
\begin{equation}
  \label{eq:opt_constraint}
  \tilde \vb_T = 
  \arg\min_{\vb \in \RR^{s}}\ 
  \frac{1}{2(n-2)}
  \sum_{i\in[n]}
  (z_i - \vb'(\xb_{i,T} - \bar\xb_T))^2 + \lambda \vb'\sbt.
\end{equation}
The optimization problem in \eqref{eq:opt_constraint} is related to
the one in \eqref{eq:opt_problem1}, however, the solution is
calculated only over the subset $T$ and $\norm{\vb_T}_1$ is replaced
with $\vb_T'\sbt$. The solution can be computed in a closed form as
\begin{equation}
  \label{eq:solution:restricted}
\begin{aligned}  
  \tilde \vb_T
  & = 
\rbr{
\hSt + \frac{n_1n_2}{n(n-2)}\hmt\hmt'
}^{-1}
\rbr{\frac{n_1n_2}{n(n-2)}\hmt - \lambda\sbt}\\
  & = 
\rbr{
\hSti - 
\frac{n_1n_2}{n(n-2)}
\frac{\hSti\hmt\hmt'\hSti}
{1+\frac{n_1n_2}{n(n-2)}\hmt'\hSti\hmt}
}
\rbr{\frac{n_1n_2}{n(n-2)}\hmt - \lambda\sbt}\\
  & = 
    \frac{n_1n_2}{n(n-2)}
    \frac{1 + \lambda\hmt'\hSti\sbt}
      {1+\frac{n_1n_2}{n(n-2)}\hmt'\hSti\hmt}
    \hSti\hmt -
    \lambda\hSti\sbt.
\end{aligned}
\end{equation}
The solution $\tilde\vb_T$ is unique, since the matrix $\hSt$ is
positive definite with probability $1$.

The following result establishes that the solution to the auxiliary
oracle optimization problem \eqref{eq:opt_constraint} satisfies
$\sgn(\tilde\vb_T) = \sbt$ with high probability, under the conditions
of Theorem~\ref{thm:main:sda}.

\begin{lemma}
  \label{lem:convergence}
  Under the assumption that the conditions of
  Theorem~\ref{thm:main:sda} are satisfied, $\sgn(\tilde\vb_T) = \sbt$
  and $\sgn(\hat\bt) = \sbt$ with probability at least $1 -
  \Ocal\rbr{\log^{-1}(n)}$.
\end{lemma}
The proof Lemma \ref{lem:convergence} relies on a careful
characterization of the deviation of the following quantities
$\hsnbt$, $\hmt'\hSti\sgn(\hat \bt)$, $\hSti\hmt$ and
$\hSti\sgn(\hat\bt)$ from their expected values.

Using Lemma~\ref{lem:convergence}, we have that $\tilde \vb_T$ defined
in \eqref{eq:solution:restricted} satisfies $\tilde \vb_T =
\hat\vb_T$.  Next, we show that $\hat\vb = (\tilde\vb_T',\zero')'$ is
a solution to \eqref{eq:opt_problem1} under the conditions
of Theorem~\ref{thm:main:sda}.

\begin{lemma}
  \label{lem:dual_certificate}
  Assuming that the conditions of Theorem~\ref{thm:main:sda} are
  satisfied, we have that $\hat\vb = (\tilde\vb_T',\zero')'$ is a solution
  to \eqref{eq:opt_problem1} with probability at least $1 -
  \Ocal\rbr{\log^{-1}(n)}$.
\end{lemma}

The proof of Theorem~\ref{thm:main:sda} will be complete once we show
that $\hat\vb =(\tilde\vb_T', \zero')'$ is the unique solution. We
proceed as in the proof of Lemma~$1$ in \cite{Wain:09a}.  Let $\check
\vb$ be another solution to the optimization problem
in~\eqref{eq:opt_problem1} satisfying the KKT condition
\[
\rbr{
\Sbb + \frac{n_1n_2}{n(n-2)}\hat\mub\hat\mub'
}
\check\vb - \frac{n_1n_2}{n(n-2)}\hat\mub 
+ \lambda\hat\qb = \zero
\]
for some subgradient $\hat\qb \in \partial \norm{\check\vb}_1$. Given
the subgradient $\hat\qb$, any optimal solution needs to satisfy the
complementary slackness condition $\hat \qb'\check\vb =
\norm{\check\vb}_1$, which holds only if $\check v_j = 0$ for all $j$
such that $|\hat q_j| < 1$. In the proof of
Lemma~\ref{lem:dual_certificate}, we established that $|\hat q_j| < 1$
for $j \in N$. Therefore, any solution to \eqref{eq:opt_problem1} has
the same sparsity pattern as $\hat\vb$. Uniqueness now follows since
$\tilde\vb_T$ is the unique solution of \eqref{eq:opt_constraint} when
constrained on the support set $T$.

\section{Lower Bound}
\label{sec:lower-bound}

Theorem~\ref{thm:main:sda} provides sufficient conditions for the SDA
estimator to reliably recover the true set $T$ of nonzero elements of
the discriminant direction $\betab$.  In this section, we provide
results that are of complementary nature. More specifically, we
provide necessary conditions that must be satisfied for any procedure
to succeed in reliable estimation of the support set $T$. Thus, we
focus on the information theoretic limits in the context of
high-dimensional discriminant analysis.

We denote $\Psi$ to be an estimator of the support set $T$, that is,
any measurable function that maps the data $\{\xb_i, y_i\}_{i\in[n]}$
to a subset of $\{1,\ldots, p\}$.  Let $\thetab = (\mub_1, \mub_2,
\Sigmab)$ be the problem parameters and $\Theta$ be the parameter
space. We define the maximum risk, corresponding to the $0/1$ loss, as
\[
R(\Psi, \Theta) = \sup_{\thetab \in \Theta}\ \PP_{\thetab}\left[\Psi(\{\xb_i,
y_i\}_{i\in[n]}) \neq T(\thetab)\right]
\]
where $\PP_{\thetab}$ denotes the joint distribution of $\{\xb_i,
y_i\}_{i\in[n]}$ under the assumption that $\pi_1=\pi_2=\frac{1}{2}$,
and $T(\thetab) = \supp( \betab )$ (recall that $\betab =
\Sigmab^{-1}(\mub_2-\mub_1)$). Let $\Mcal(s,\Zcal)$ be the class of
all subsets of the set $\Zcal$ of cardinality $s$. We consider the
parameter space
\begin{equation}
  \label{eq:model_theta}
  \Theta(\Sigmab, \tau, s) = \!\!\!\!\!\!\!
  \bigcup_{\omega \in \Mcal(s, [p])}
\!\!\! \left\{
    \thetab = (\mub_1, \mub_2, \Sigmab)\ :\ 
    \begin{array}{l}
    \betab = \Sigmab^{-1}(\mub_2 - \mub_1),\\
    |\beta_a| \geq \tau
    \text{ if } a \in \omega,\ 
    \beta_a = 0 
    \text{ if } a \not\in \omega
    \end{array} \!\!
\right\},
\end{equation}
where $\tau > 0$ determines the signal strength.
The minimax risk is defined as 
\[
\inf_{\Psi} R(\Psi, \Theta(\Sigmab, \tau, s)).
\]
In what follows we provide a lower bound on the minimax risk.  Before
stating the result, we introduce the following three quantities that
will be used to state Theorem~\ref{thm:lower_bound}
\begin{align}
 & \varphi_{\rm close}(\Sigmab) =
  \min_{T \in \Mcal(s, [p])}\ \min_{u \in T}\ \frac{1}{p-s}\sum_{v \in [p]\bks T}
  \left( \Sigma_{uu} + \Sigma_{vv} - 2\Sigma_{uv} \right),
  \label{eq:close_ensemble:phi}\\
&  \varphi_{\rm far}(\Sigmab) = \min_{T \in \Mcal(s, [p])}\ 
  \frac{1}{\ds {{p-s} \choose s}} \sum_{T' \in \Mcal(s, [p]\bks T)} 
  \one' \Sigmab_{T\cup T', T\cup T'} \one,   \label{eq:far_ensemble:phi}
\end{align}
and 
\begin{equation}
  \label{eq:beta_min:lower_bound}
\tau_{\min} = 2\cdot \max\left(
\sqrt{\frac{\log{p-s \choose s}}{n\varphi_{\rm far}(\Sigmab)}},
\sqrt{\frac{\log(p-s+1)}{n\varphi_{\rm close}(\Sigmab)}}
\right).
\end{equation}
The first quantity measures the difficulty of distinguishing two close
support sets $T_1$ and $T_2$ that differ in only one position. The
second quantity measures the effect of a large number of support sets
that are far from the support set $T$. The quantity $\tau_{\min}$ is a
threshold for the signal strength. Our main result on minimax lower
bound is presented in Theorem \ref{thm:lower_bound}.

\begin{theorem}
  \label{thm:lower_bound}
For any $\tau < \tau_{\min}$, there exists some constant $C>0$, such that
\begin{equation*}
\inf_{\Psi}  
\sup_{\thetab \in \Theta(\Sigmab, \tau, s)}\ \PP_{\thetab}
\left[\Psi(\{\xb_i,y_i\}_{i\in[n]}) \neq T(\thetab)\right] \geq
C > 0.
\end{equation*}
\end{theorem}
Theorem \ref{thm:lower_bound} implies that for any estimating
procedure, whenever $\tau < \tau_{\min}$, there exists some
distribution parametrized by $\thetab \in \Theta(\Sigmab, \tau, s)$
such that the probability of incorrectly identifying the set
$T(\thetab)$ is strictly bounded away from zero.  To better understand
the quantities $\varphi_{\rm close}(\Sigmab)$ and $\varphi_{\rm
  far}(\Sigmab)$, we consider a special case when $\Sigmab = \Ib$. In
this case both quantities simplify a lot and we have $\varphi_{\rm
  close}(\Ib) = 2$ and $\varphi_{\rm far}(\Ib) = 2s$. From
Theorem~\ref{thm:lower_bound} and Theorem~\ref{thm:main:sda}, we see
that the SDA estimator is able to recover the true support set $T$
using the optimal number of samples (up to an absolute constant) over
the parameter space
  \[
  \Theta(\Sigmab, \tau_{\min}, s) \cap \{\thetab\ :\ \snbt \leq M \}
  \] 
  where $M$ is a fixed constant and $\Lambda_{\min}(\St)$ is bounded
  from below. This result will be further illustrated by numerical
  simulations in $\S$\ref{sec:simulation}.

\section{Exhaustive Search Decoder}
\label{sec:exhaustive_search}

In this section, we analyze an exhaustive search procedure, which
evaluates every subset $T'$ of size $s$ and outputs the one with the
best score. Even though the procedure cannot be implemented in
practice, it is a useful benchmark to compare against and it provides
deeper theoretical insights into the problem.

For any subset $T' \subset [p]$, we define
\begin{equation}
  \label{eq:f_decoder}
  f(T') = \min_{\ub \in \RR^{|T'|}} \left\{
    \ub'\hat \Sbb_{T'T'}\ub
    \ :\ 
    \ub'\hat\mub_{T'} = 1
    \right) = 
\min_{T' \subset [p]\ :\ |T'| = s}\ 
 \frac{1}{\hat \mub_{T'}'\Sbb^{-1}_{T'T'}\hat \mub_{T'}'}. \nonumber
\end{equation}
The exhaustive search procedure outputs the support set $\hat T$ that
minimizes $f(T')$ over all subsets $T'$ of size $s$,
\begin{equation*}
  \hat T = \argmin_{T' \subset [p]\ :\ |T'| = s} f(T')
= \argmax_{T' \subset [p]\ :\ |T'| = s} 
     \hat \mub_{T'}'\Sbb^{-1}_{T'T'}\hat \mub_{T'}.
\end{equation*}
Define $g(T') = \hat \mub_{T'}'\Sbb^{-1}_{T'T'}\hat
\mub_{T'}$. In order to show that the exhaustive search procedure
identifies the correct support set $T$, we need to show that with high probability
$g(T) > g(T')$ for any other set $T'$ of size $s$. The next result
gives sufficient conditions for this to happen. We first introduce
some additional notation. Let $A_1 = T\cap T'$, $A_2 = T\bks T'$ and
$A_3 = T' \bks T$. We  define the following quantities
  \begin{align*}
    a_1(T') &= \tmp{A_1}'\tSip{A_1 A_1}\tmp{A_1},\\
    a_2(T') &= \tmp{A_2\mid A_1}'\tSip{A_2 A_2\mid A_1}
               \tmp{A_2\mid A_1},\\
    a_3(T') &= \tmp{A_3\mid A_1}'\tSip{A_3 A_3\mid A_1}
               \tmp{A_3\mid A_1},
  \end{align*}
  where   $\tmp{A_2\mid A_1} 
    = \tmp{A_2} - 
    \tSp{A_2A_1}\tSip{A_1A_1}\tmp{A_1}$
and
$ \tSp{A_2A_2\mid A_1} = 
  \tSp{A_2A_2} - 
  \tSp{A_2A_1}\tSip{A_1A_1}\tSp{A_2A_1}$.
   The quantities $ \tmp{A_3\mid A_1} $ and $ \tSp{A_3A_3\mid A_1}$ are
defined similarly.
\begin{theorem}
  \label{thm:exhaustive_search}
  Assuming that for all $T' \subseteq [p]$ with $|T'| = s$ and $T' \neq
  T$ the following holds
\begin{equation}
\label{eq:exhaustive_identification}
  \begin{aligned}
a_2(T') - \rbr{1+C_1\sqrt{\Gamma_{n,p,s,k}}}a_3(T') 
&\geq
C_2\sqrt{
\rbr{1 \vee a_1(T')}
a_2(T')
\Gamma_{n,p,s,k}
}  \\
&\qquad +C_3\rbr{1 \vee a_1(T')}\Gamma_{n,p,s,k}
  \end{aligned},
\end{equation}
where $|T'\cap T|=k$, 
$\Gamma_{n,p,s,k} = n^{-1}\log\rbr{{p-s \choose
    s-k}{s \choose k}s\log(n)} $
and $C_1, C_2, C_3$ are constants independent of the problem
parameters, we have
$\PP[\hat T \neq T] = \Ocal(\log^{-1}(n))$.
\end{theorem}

The condition in \eqref{eq:exhaustive_identification} allows the exhaustive search
decoder to distinguish between the sets $T$ and $T'$ with high
probability. Note that the Mahalanobis distance  decomposes as
$
g(T) = \hmp{A_1}'\hSip{A_1A_1}\hmp{A_1} +
\tilde\mub_{A_2\mid A_1}'\hSip{A_2A_2\mid A_1}\tilde\mub_{A_2\mid A_1}
$
where 
$
\tilde\mub_{A_2\mid A_1} = \hmp{A_2} -
\hSp{A_2A_1}\hSip{A_1A_1}\hmp{A_1}
$ and $
\hSp{A_2A_2\mid A_1} = \hSp{A_2A_2} -
\hSp{A_2A_1}\hSip{A_1A_1}\hSp{A_1A_2}$,
 and similarly
$
g(T') = \hmp{A_1}'\hSip{A_1A_1}\hmp{A_1} +
\tilde\mub_{A_3\mid A_1}'\hSip{A_3A_3\mid A_1}\tilde\mub_{A_3\mid A_1}.
$
Therefore $g(T) > g(T')$ if
$
\tilde\mub_{A_2\mid A_1}'\hSip{A_2A_2\mid A_1}\tilde\mub_{A_2\mid  A_1}
>
\tilde\mub_{A_3\mid A_1}'\hSip{A_3A_3\mid A_1}\tilde\mub_{A_3\mid A_1}
$. With infinite amount of data, it would be sufficient that $a_2(T')
> a_3(T')$. However, in the finite-sample setting, condition
\eqref{eq:exhaustive_identification} ensures that the separation is
big enough.
If $\Xb_T$ and $\Xb_N$ are independent, then the expression
\eqref{eq:exhaustive_identification} can be simplified by dropping the
second term on the left hand side.

Compared to the result of
Theorem~\ref{thm:main:sda},  the exhaustive search
procedure does not require the covariance matrix to satisfy the
irrepresentable condition given in \eqref{eq:assum:irrepresentable}.

\section{Implications of Our Results}
\label{sec:consequence}

In this section, we give some implications of our results.  We start
with the case when the covariance matrix $\Sigmab = \Ib$. The same
implications hold for other covariance matrices that satisfy
$\Lambda_{\min}(\Sigmab) \geq C > 0$ for some constant $C$ independent
of $(n,p,s)$. We first illustrate a regime where the SDA estimator is
optimal for the problem of identifying the relevant variables. This is
done by comparing the results in Theorem~\ref{thm:main:sda} to those
of Theorem~\ref{thm:lower_bound}. Next, we point out a regime where
there exists a gap between the sufficient and necessary conditions of
Theorem~\ref{thm:lower_bound} for both the exhaustive search decoder
and the SDA estimator.  Throughout the section, we assume that $s =
o(\min(n, p))$.

When $\Sigmab = \Ib$, we have that $\bt =
\mt$. Let
\[
\mu_{\min} = \min_{a \in T} |\mub_{T}|.
\] 
Theorem~\ref{thm:lower_bound} gives a lower bound on $\mu_{\min}$ as
\[
\mu_{\min} \gtrsim \sqrt{\frac{\log(p-s)}{n}}.
\]
If some components of the vector $\mt$ are smaller in absolute value
than $\mu_{\min}$, no procedure can reliably recover the support.  We
will compare this bound with sufficient conditions given in Theorems
\ref{thm:main:sda} and
\ref{thm:exhaustive_search}.

First, we assume that $\norm{\mt}_2^2 = C$ for some constant
$C$. Theorem~\ref{thm:main:sda} gives that $\mu_{\min} \gtrsim
\sqrt{\small\frac{\log(p-s)}{n}}$ is sufficient for the SDA estimator
to consistently recover the relevant variables when $n \gtrsim
s\log(p-s)$. This effectively gives $\mu_{\min} \gtrsim s^{-1/2}$,
which is the same as the necessary condition of
Theorem~\ref{thm:lower_bound}.

Next, we investigate the condition in
\eqref{eq:exhaustive_identification}, which is sufficient for the
exhaustive search procedure to identify the set $T$. Let $T' \subset
[p]$ be a subset of size $s$. Then, using
the notation of Section~\ref{sec:exhaustive_search}, 
\[
a_1(T') = \norm{\tmp{A_1}}_2^2,\ 
a_2(T') = \norm{\tmp{A_2}}_2^2,
\text{ and }
a_3(T') = 0.
\]
Now, if $|T'\cap T| = s-1$ and $T'$ does not contain a smallest
component of $\mt$, \eqref{eq:exhaustive_identification} simplifies to
$\mu_{\min} \gtrsim \sqrt{\small \frac{\log(p-s)}{n}}$, since
$\norm{\tmp{A_1}}_2^2 \leq \norm{\mt}_2^2 = C$.  This shows that both
the SDA estimator and the exhaustive search procedure can reliably
detect signals at the information theoretic limit in the case when the
norm of the vector $\mt$ is bounded and $\mu_{\min} \gtrsim s^{-1/2}$.
However, when the norm of the vector $\mt$ is not bounded by a
constant, for example, $\mu_{\min} = C'$ for some constant $C'$,
Theorem~\ref{thm:lower_bound} gives that at least $n \gtrsim
\log(p-s)$ data points are needed, while $n \gtrsim s\log(p-s)$ is
sufficient for correct recovery of the support set $T$. This situation
is analogous to the known bounds on the support recovery in the sparse
linear regression setting \citep{Wain:09b}.

Next, we show that the largest eigenvalue of a covariance matrix
$\Sigmab$ can diverge, without affecting the sample size required for
successful recovery of the support set $T$. Let $\Sigmab =
(1-\gamma)\Ib_p + \gamma\one_p\one_p'$ for $\gamma \in [0, 1)$.  We
have $\Lambda_{\max}(\Sigmab) = 1 + (p-1)\gamma$, which diverges to
infinity for any fixed $\gamma$ as $p\rightarrow\infty$. Let $T=[s]$
and set $\bt = \beta \one_T$. This gives $\mt =
\beta(1+\gamma(s-1))\one_T$ and $\mub_{N} = \gamma\beta s\one_N$. A
simple application of the matrix inversion formula gives
\[
\Sigmab^{-1}_{TT} = (1-\gamma)^{-1}\Ib_s 
- \frac{\gamma}{(1-\gamma)(1+\gamma(s-1))}\one_T\one_T'.
\]

A lower bound on $\beta$ is obtained from
Theorem~\ref{thm:lower_bound} as $\beta \geq
\sqrt{\frac{2}{1-\gamma} \frac{\log(p-s)}{n}}$. This follows
from a simple calculation that establishes $ \varphi_{\rm
  close}(\Sigmab) = 2(1-\gamma) $ and $ \varphi_{\rm far}(\Sigmab) =
2s(1-\gamma) + (2s)^2\gamma.  $

Sufficient conditions for the SDA estimator
follow from Theorem \ref{thm:main:sda}. A
straightforward calculation shows that
\[
\sigma_{a|T} = \frac{(1-\gamma)(1+\gamma s)}{1+\gamma(s-1)},\ 
\Lambda_{\min}(\Sigmab) = 1-\gamma,\ 
\norm{\Sti\sbt}_\infty = \frac{1}{1+\gamma(s-1)}.
\]
This gives that $\beta \geq K \sqrt{\frac{\log(p-s)}{(1-\gamma)n}}$
(for $K$ large enough) is sufficient for recovering the set $T$,
assuming that $\snbt = \Ocal(1)$. This matches the lower bound,
showing that the maximum eigenvalue of the covariance matrix $\Sigmab$
does not play a role in characterizing the behavior of the SDA
estimator.

\section{Simulation Results}
\label{sec:simulation}

In this section, we conduct several simulations to illustrate the
finite-sample performance of our results.  Theorem~\ref{thm:main:sda}
describes the sample size needed for the SDA estimator to recover the
set of relevant variables. We consider the following three scalings
for the size of the set $T$:
\begin{enumerate}
\item fractional power sparsity, where $s = \lceil 2p^{0.45}
  \rceil$
\item sublinear sparsity, where $s = \lceil 0.4p/\log(0.4p)
  \rceil$, and
\item linear sparsity, where $s = \lceil 0.4p \rceil$. 
\end{enumerate}
For all three scaling regimes, we set the sample size as 
\[
n = \theta s \log(p)
\]
where $\theta$ is a control parameter that is varied.  We
investigate how well can the SDA estimator recovers
the true support set $T$ as the control parameter $\theta$ varies.

\begin{figure}[ht!]
  \begin{subfigure}[b]{\columnwidth}
    \centering
    \includegraphics[width=\textwidth]{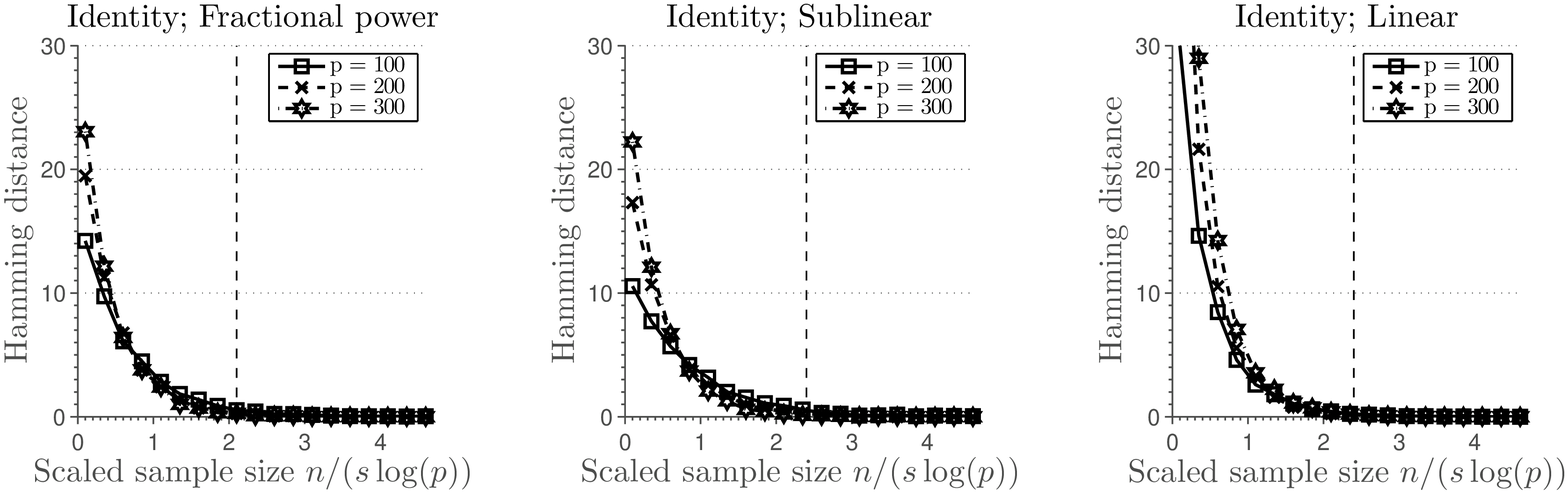}    
    \label{fig:identity_sda}
   \end{subfigure}%
   \\[-0.0cm]

   \caption{(The SDA Estimator) Plots of the rescaled sample size $n /
     (s\log(p))$ versus the Hamming distance between $\hat T$ and $T$
     for identity covariance matrix $\Sigmab = \Ib_p$ (averaged over
     200 simulation runs).  Each subfigure shows three curves,
     corresponding to the problem sizes $p \in \{100, 200, 300\}$. The
     first subfigure corresponds to the fractional power sparsity
     regime, $s=2p^{0.45}$, the second subfigure corresponds to the
     sublinear sparsity regime $s = 0.4p/\log(0.4p)$, and the third
     ssubfigure corresponds to the linear sparsity regime $s =
     0.4p$. Vertical lines denote a scaled sample size at which the
     support set $T$ is recovered correctly.  }
  \label{fig:identity}
\end{figure}

\begin{figure}[ht!]

  \begin{subfigure}[b]{\columnwidth}
    \centering
    \includegraphics[width=\textwidth]{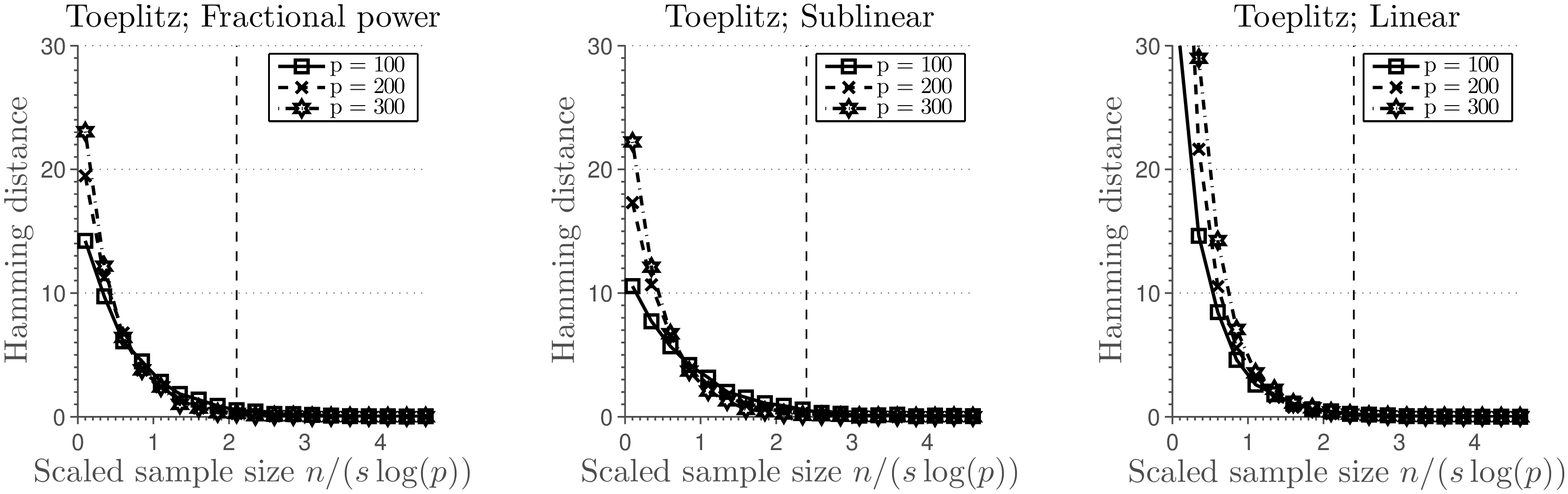}    
    \label{fig:toeplitz_sda}
   \end{subfigure}%
   \\[0cm]

   \caption{(The SDA Estimator) Plots of the rescaled sample size $n /
     (s\log(p))$ versus the Hamming distance between $\hat T$ and $T$
     for the Toeplitz covariance matrix $\Sigmab_{TT}$ with $\rho =
     0.1$ (averaged over 200 simulation runs).  Each subfigure shows
     three curves, corresponding to the problem sizes $p \in \{100,
     200, 300\}$. The first subfigure corresponds to the fractional
     power sparsity regime, $s=2p^{0.45}$, the second subfigure
     corresponds to the sublinear sparsity regime $s =
     0.4p/\log(0.4p)$, and the third subfiguren corresponds to the
     linear sparsity regime $s = 0.4p$. Vertical lines denote a scaled
     sample size at which the support set $T$ is recovered correctly.
   }
  \label{fig:toeplitz}
\end{figure}

\begin{figure}[ht!]
  \centering
  
  \begin{subfigure}[b]{\columnwidth}
    \centering
    \includegraphics[width=\textwidth]{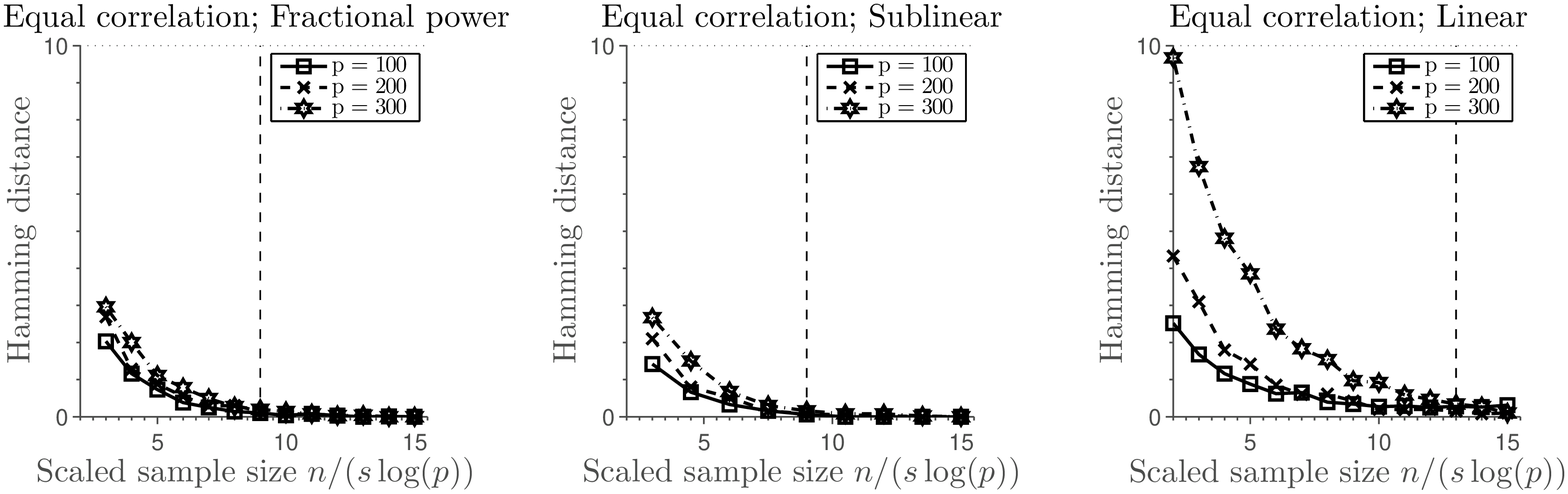}    
    \label{fig:equalCor_sda}
   \end{subfigure}%
   \\[0cm]

   \caption{(The SDA Estimator) Plots of the rescaled sample size $n /
     (s\log(p))$ versus the Hamming distance between $\hat T$ and $T$
     for equal correlation matrix $\Sigmab_{TT}$ with $\rho = 0.1$
     (averaged over 200 simulation runs). Each subfigure shows three
     curves, corresponding to the problem sizes $p \in \{100, 200,
     300\}$. The first subfigure corresponds to the fractional power
     sparsity regime, $s=2p^{0.45}$, the second subfigure corresponds
     to the sublinear sparsity regime $s = 0.4p/\log(0.4p)$, and the
     third subfigure corresponds to the linear sparsity regime $s =
     0.4p$. Vertical lines denote a scaled sample size at which the
     support set $T$ is recovered correctly.  }
  \label{fig:equalCor}
\end{figure}

We set $\PP[Y=1] = \PP[Y=2] = \frac{1}{2}$, $\Xb|Y=1 \sim \Ncal(\mub,
\Sigmab)$ and without loss of generality $\Xb|Y=2 \sim \Ncal(\zero,
\Sigmab)$. We specify the vector $\mub$ by choosing the set $T$ of
size $|T| = s$ randomly, and for each $a \in T$ setting $\mu_a$ equal
to $+1$ or $-1$ with equal probability, and $\mu_a = 0$ for all
components $a \not\in T$. We specify the covariance matrix $\Sigmab$
as
\begin{equation*}
  \Sigmab = 
  \left(
    \begin{array}{cc}
      \Sigmab_{TT} & \zero \\
      \zero & \Ib_{p-s}
    \end{array}
  \right)
\end{equation*}
so that $\betab = \Sigmab^{-1}\mub = (\betab_T', \zero')'$.  We
consider three cases for the block component $\Sigmab_{TT}$:
\begin{enumerate}
\item identity matrix, where $\Sigmab_{TT} = \Ib_s$,
\item Toeplitz matrix, where $\Sigmab_{TT} = [\Sigma_{ab}]_{a,b\in T}$
  and $\Sigma_{ab} = \rho^{|a-b|}$ with $\rho = 0.1$, and
\item equal correlation matrix, where $\Sigma_{ab} = \rho$ when $a\neq
  b$ and $\sigma_{aa} = 1$. 
\end{enumerate}
Finally, we set the penalty parameter 
$\lambda = \lambda_{\rm SDA}$ as
\begin{equation*}
\lambda_{\rm SDA} = 0.3 \times 
\rbr{1+\snbt/4}^{-1}
       \sqrt{
         \rbr{1\vee\snbt}
         \frac{\log\rbr{p-s}}{n}
       }       
\end{equation*}
for all cases. We also tried several different constants and found
that our main results on high dimensional scalings are insensitive to
the choice of this constant.  For this choice of $\lambda$,
Theorem~\ref{thm:main:sda} predicts that the set $T$ will be recovered
correctly. For each setting, we report the Hamming distance between
the estimated set $\hat T$ and the true set $T$,
\begin{equation*}
h(\hat T, T) = |(\hat T \bks T)\cup(T \bks \hat T)|,  
\end{equation*}
averaged over $200$ independent simulation runs.

Figure~\ref{fig:identity} plots the Hamming distance against the
control parameter $\theta$, or the rescaled number of samples. Here
the Hamming distance between $\hat T$ and $T$ is calculated by
averaging $200$ independent simulation runs.  There are three
subfigures corresponding to different sparsity regimes (fractional
power, sublinear and linear sparsity), each of them containing three
curves for different problem sizes $p \in \{100, 200, 300\}$.
Vertical line indicates a threshold parameter $\theta$ at which the
set $T$ is correctly recovered. If the parameter is smaller than the
threshold value, the recovery is poor. Figure~\ref{fig:toeplitz} and
Figure~\ref{fig:equalCor} show results for two other cases, with
$\Sigmab_{TT}$ being a Toeplitz matrix with parameter $\rho = 0.1$ and
the equal correlation matrix with $\rho=0.1$. To illustrate the effect
of correlation, we set $p=100$ and generate the equal correlation
matrices with $\rho \in \{0, 0.1, 0.3, 0.5, 0.7, 0.9 \}$. Results are
given in Figure~\ref{fig:equalCor_2}.

\begin{figure}[t]
  \centering
  
  \begin{subfigure}[b]{\columnwidth}
    \centering
    \includegraphics[width=\textwidth]{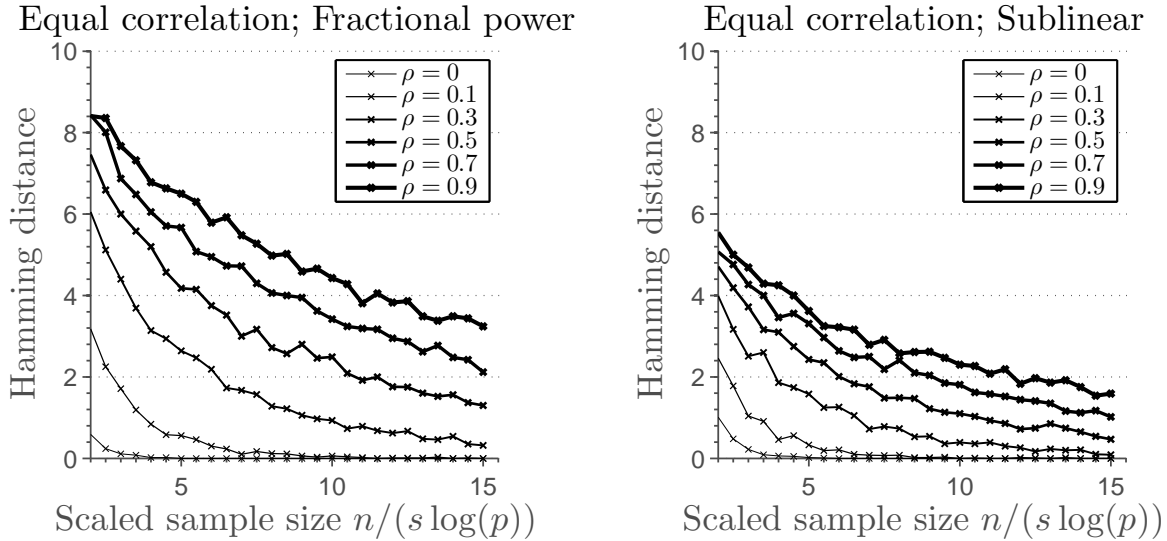}    
    \label{fig:equalCor_sda_2}
   \end{subfigure}%
   \\[0cm]

   \caption{ (The SDA Estimator) Plots of the rescaled sample size $n
     / (s\log(p))$ versus the Hamming distance between $\hat T$ and
     $T$ for equal correlation matrix $\Sigmab_{TT}$ with $\rho \in
     \{0, 0.1, 0.3, 0.5, 0.7, 0.9\}$ (averaged over 200 simulation
     runs). The ambient dimension is set as $p=100$.  The first
     subfigure corresponds to the fractional power sparsity regime,
     $s=2p^{0.45}$ and the second subfigure corresponds to the
     sublinear sparsity regime $s = 0.4p/\log(0.4p)$. }
  \label{fig:equalCor_2}
\end{figure}

\section{Discussion}
\label{sec:discussion}

In this paper, we address the problem of variable selection in
high-dimensional discriminant analysis problem. The problem of
reliable variable selection is important in many scientific areas
where simple models are needed to provide insights into complex
systems. Existing research has focused primarily on establishing
results for prediction consistency, ignoring feature selection. We
bridge this gap, by analyzing the variable selection performance of
the SDA estimator and an exhaustive search decoder. We establish
sufficient conditions required for successful recovery of the set of
relevant variables for these procedures. This analysis is complemented
by analyzing the information theoretic limits, which provide necessary
conditions for variable selection in discriminant analysis. From these
results, we are able to identify the class of problems for which the
computationally tractable procedures are optimal. In this section, we
discuss some implications and possible extensions of our results.

\subsection{Theoretical Justification of the ROAD and Sparse Optimal
  Scaling Estimators}

In a recent work, \cite{mai2012note} show that the SDA estimator is
numerically equivalent to the ROAD estimator 
proposed by \cite{wu2009sparse, fan2010road} and the sparse optimal
scaling estimator proposed by \cite{clemmensen2011sparse}. More
specifically, all these three methods have the same regularization
paths up to a constant scaling.  This result allows us to apply the
theoretical results in this paper to simultaneously justify the
optimal variable selection performance of the ROAD and sparse optimal
scaling estimators.

\subsection{Risk Consistency}

The results of Theorem \ref{thm:main:sda} can be
used to establish risk consistency of the SDA estimator.
Consider the following classification rule
\begin{equation}
  \label{eq:rule_class}
  \hat y(\xb) = \left\{ 
    \begin{array}{cl}
      1 & \text{if } g(\xb; \hat \vb) = 1 \\
      2 & \text{otherwise }
  \end{array}
  \right.
\end{equation}
where $g(\xb; \hat \vb) = I\sbr{
    \hat\vb'(\xb - (\hat\mub_1 + \hat\mub_2)/2) > 0
  }$
with  $\hat \vb = \vsda$.
Under the assumption that $\betab = (\bt', \zero')'$, 
the risk (or the error rate) of the Bayes rule
defined in \eqref{bayes_rule} is 
$
R_{\rm opt} = \Phi\rbr{ 
-\sqrt{\mt'\Sti\mt}/2
}$
,
where $\Phi$ is the cumulative distribution function of a standard
Normal distribution.  We will compare the risk of the SDA estimator
against this Bayes risk. 

Recall the setting introduced in 
$\S$\ref{sec::mainresult},  conditioning on the data
points $\{\xb_i,y_i\}_{i\in[n]}$, the conditional error rate  is
\begin{equation}
  \label{eq:error_rate}
  \begin{aligned}
  R(\hat\wb) =
  \frac{1}{2}
  \sum_{i \in \{ 1, 2\}}
  \Phi\rbr{
    \frac{-\hat\vb'(\mub_{i} - \hat\mub_{i}) -
      \hat\vb'\hat\mub/2}{\sqrt{\hat\vb'\Sigmab\hat\vb}}
  }.    
  \end{aligned}
\end{equation}
Let $ r_n = \lambda\norm{\bt}_1$ and $q_n = \sbt'\St\sbt$.
We have the following result  on risk consistency.

\begin{corollary}
\label{thm:risk}
Let $\hat \vb = \vsda$.
We assume that the conditions of 
Theorem~\ref{thm:main:sda} hold
with
\begin{equation}
    \label{eq:sample_size:risk}
    n \asymp K(n) \rbr{\max_{a\in N}\sigma_{a|T}}
           \Lambda_{\min}^{-1}(\St)
           s\log\rbr{(p-s)\log(n)},
\end{equation}
where $K(n)$ could potentially scale with $n$, and $\snbt \geq C > 0$. 
Furthermore,  we assume that $r_n \xrightarrow{n\rightarrow\infty} 0$. Then
\[
R(\hat\wb) = 
\Phi\rbr{
-\frac{\nbt}{2}
\frac{\rbr{1+\Ocal_P\rbr{r_n}}}
{\sqrt{1+\Ocal_P\rbr{r_n
\vee\frac{\lambda_0^2q_n}{\snbt}}}}
}.
\]
\end{corollary}

First, note that
$
{\norm{\bt}_1}/{\nbt} = o\rbr{
\sqrt{{K(n)s}/{\Lambda_{\min}(\St)}}
}
$
is sufficient for $r_n \xrightarrow{n\rightarrow\infty} 0$.
Under the conditions of Theorem~\ref{thm:risk}, we have that 
$
{\snbt}/\rbr{\lambda_0^2q_n} = 
\Ocal\rbr{
K(n)s/\rbr{\Lambda_{\min}(\St)q_n}
}
=
\Ocal\rbr{K(n)}$.
Therefore, if $K(n) \xrightarrow{n\rightarrow\infty}\infty$ and
$
K(n) \geq C s\snbt/\rbr{\Lambda_{\min}(\St)\norm{\bt}_1^2}
$
 we have
\[
R(\hat\wb) = 
\Phi\rbr{
  -\frac{\nbt}{2}
  \rbr{1+\Ocal_P\rbr{r_n}}
}
\]
and $R(\hat\wb) - R_{\rm opt} \rightarrow_P 0$.
If in addition
$$
\nbt\norm{\bt}_1 = o\rbr{
\sqrt{{K(n)s}/{\Lambda_{\min}(\St)}}
},
$$
then 
$
R(\hat\wb)/R_{\rm opt} \rightarrow_P 1,
$ using Lemma~1 in \cite{shao2011lda}.

The above discussion shows that the conditions of
Theorem~\ref{thm:main:sda} are sufficient for establishing risk
consistency.  We conjecture that substantially less restrictive
conditions are needed to establish risk consistency results. Exploring
such weaker conditions is beyond the scope of this paper.

\subsection{Approximate sparsity}
\label{sec:approx_sparse}

Thus far, we were discussing estimation of discriminant directions
that are exactly sparse. However, in many applications it may be the
case that the discriminant direction $\betab = (\bt', \betab_N')' =
\Sigmab^{-1}\mub$ is only approximately sparse, that is, $\betab_N$ is
not equal to zero, but is small. In this section, we briefly discuss
the issue of variable selection in this context. 

In the approximately sparse setting, since $\betab_N \neq \zero$, a
simple calculation gives
\begin{equation}
  \label{eq:approx:1}
  \bt = \Sigmab_{TT}^{-1}\mub_{T} - \Sti\Sigmab_{TN}\betab_N
\end{equation}
and
\begin{equation}
  \label{eq:approx:2}
  \mub_{N} =
  \Sigmab_{NT}\Sigmab_{TT}^{-1}\mub_{T} +
  \rbr{\Sigmab_{NN}- \Sigmab_{NT}\Sti\Sigmab_{TN}}\betab_N.
\end{equation}
In what follows, we provide conditions under which the solution to the
population version of the SDA estimator, given in
\eqref{eq:population:optimization}, correctly recovers the support of
large entries $T$.
Let $\hat\wb = (\hat\wb_T', \zero')'$ where $\hat\wb_T$ is given as
\[
    \hat\wb_T = \pi_1\pi_2
    \frac{1+\lambda\norm{\tilde\bt}_1}
         {1 + \pi_1\pi_2\norm{\tilde\bt}_{\St}^2}
    \tilde\bt - 
    \lambda\Sigmab_{TT}^{-1}\sgn(\tilde\bt)
\]
with $\tilde\bt=\Sti\mt$. We will show that $\hat\wb$ is the solution
to \eqref{eq:population:optimization}.

We again define $ \beta_{\min} = \min_{a \in T} |\beta_a|$.  Following
a similar argument as the proof of Theorem \ref{thm:sda_population},
we have that $\sgn(\hat\wb_T) = \sgn\rbr{\tilde\bt}$ holds if
$\tilde\bt$ satisfies 
\begin{equation}
\pi_1\pi_2
\frac{1 +
  \lambda\norm{\tilde\bt}_1}{1 + \pi_1\pi_2\norm{\tilde\bt}_{\St}^2}
   \beta_{\min} > 
    \lambda \norm{\Sti\sgn\rbr{\tilde\bt}}_\infty.  \label{eq:assum:beta_min:approx}
\end{equation}
In the approximate sparsity setting, it is reasonable to assume that
$\Sti\Sigmab_{TN}\betab_N$ is small compared to $\tilde\bt$, which
would imply that $\sbt = \sgn\rbr{\tilde\bt}$ using
\eqref{eq:approx:1}.  Therefore, under suitable assumptions we have
$\sgn(\hat\wb_T) = \sbt$. Next, we need conditions under which
$\hat\wb$ is the solution to \eqref{eq:population:optimization}.

Following a similar analysis as in Lemma~\ref{lem:global_solution},
the optimality condition
\[
\norm{
    \rbr{\Sigmab_{NT} + \pi_1\pi_2\mub_N\mub_T'}\hat\wb_T - \pi_1\pi_2\mub_N}_\infty
 \leq \lambda
\]
needs to hold. Let $\hat{\gamma} = \ds \frac{1 + \lambda\norm{\tilde\bt}_1}{1 +
  \pi_1\pi_2\norm{\tilde\bt}_{\St}^2}$.
Using \eqref{eq:approx:2}, the above display becomes 
\[
\norm{
 -\lambda \Sigmab_{NT}\Sti\sbt 
- \pi_1\pi_2\hat{\gamma}
\rbr{\Sigmab_{NN}- \Sigmab_{NT}\Sti\Sigmab_{TN}}\betab_N
}_\infty < \lambda.
\]
Therefore, using  the triangle inequality, the following assumption
\[
\pi_1\pi_2
\hat{\gamma}\cdot\norm{
\rbr{\Sigmab_{NN}- \Sigmab_{NT}\Sti\Sigmab_{TN}}\betab_N
}_\infty < \alpha\lambda,
\]
in addition to \eqref{eq:assum:irrepresentable} and
\eqref{eq:assum:beta_min:approx}, is sufficient for $\hat\wb$ to
recover the set of important variables $T$.

The above discussion could be made more precise and extended to the
sample SDA estimator in \eqref{eq:opt_problem1}, by following the
proof of Theorem~\ref{thm:main:sda}. This is beyond the scope of the
current paper and will be left as a future investigation.

\appendix

\section{Proofs Of Main Results}

In this section, we collect proofs of results given in the main text.
We will use $C, C_1, C_2, \ldots$ to denote generic constants that do
not depend on problem parameters. Their values may change from line to
line.

Let
\begin{equation}
\label{eq:acal} 
\Acal = 
\Ecal_{n}\cap\Ecal_1(\log^{-1}(n))\cap\Ecal_2(\log^{-1}(n))
\cap\Ecal_3(\log^{-1}(n))\cap\Ecal_4(\log^{-1}(n)),
\end{equation}
where $\Ecal_{n}$ is defined in \eqref{eq:event_n}, $\Ecal_1$ in
Lemma~\ref{lem:hmt_hsti_hmt_ratio}, $\Ecal_2$ in
Lemma~\ref{lem:hmt_sti_hmt}, $\Ecal_3$ in \eqref{eq:event_e3}, and
$\Ecal_4$ in \eqref{eq:event_e4}. We have that $\PP[\Acal] \geq 1 -
\Ocal\rbr{\log^{-1}(n)}$.

\subsection{Proofs of Results in Section~\ref{sec:sda}}

\begin{proof}[Proof of Lemma~\ref{lem:global_solution}]
  From the KKT conditions given in \eqref{eq:KKT}, we have that
  $\hat \wb = (\hat\wb_T', \zero')'$ is a solution to the problem in
  \eqref{eq:population:optimization} if and only if
  \begin{equation}
  \label{eq:kkt:rewritten}
  \begin{aligned}
  &\rbr{\St + \pi_1\pi_2\mt\mt'}\hat\wb_T - \pi_1\pi_2\mt + 
          \lambda \sgn(\hat\wb_T) = \zero\\
  &\norm{
    \rbr{\Sigmab_{NT} + \pi_1\pi_2\mub_N\mub_T'}\hat\wb_T - \pi_1\pi_2\mub_N}_\infty
          \leq \lambda \\
  \end{aligned}    
  \end{equation}  
  By construction, $\hat\wb_T$ satisfy the first equation. Therefore,
  we need to show that the second one is also satisfied. 
  Plugging in the explicit form of $\hat\wb_T$ into the second
  equation and using \eqref{eq:beta_t_mu_n}, after some algebra we
  obtain that
  \[
     \norm{\Sigmab_{NT}\Sigmab_{TT}^{-1}\sbt}_\infty \leq 1
  \]
  needs to be satisfied. The above display is satisfied
  with strict inequality
  under the assumption in
  \eqref{eq:assum:irrepresentable}. 
\end{proof}

\begin{proof}[Proof of Lemma~\ref{lem:convergence}]
Throughout the proof, we will work on the event $\Acal$ defined in
\eqref{eq:acal}.

Let $a\in T$ be such that $\tilde v_a > 0$, noting that the case when
$\tilde v_a < 0$ can be handled in a similar way. Let
\begin{gather*}
  \delta_1 = \hmt'\hSti\sbt-\mt'\Sti\sbt, \quad
  \delta_2 = \eb_a'\hSti\hmt -\eb_a'\Sti\mt, \\
  \delta_3 = \eb_a'\hSti\sbt - \eb_a'\Sti\sbt, \quad
  \delta_4 = \hsnbt - \snbt,\text{ and}\\
 \delta_5 = \nr - \pi_1\pi_2.
\end{gather*}
Furthermore, let
\[
  \hat\gamma = 
    \frac{n_1n_2}{n(n-2)}
    \frac{1 + \lambda\hmt'\hSti\sbt}
      {1+\frac{n_1n_2}{n(n-2)}\hmt'\hSti\hmt}
\qquad\text{and}\qquad
  \gamma =   
    \frac{\pi_1\pi_2(1 + \lambda\norm{\betab_T}_1)}
      {1+\pi_1\pi_2\snbt}.
\]

For sufficiently large $n$, on the event $\Acal$, together with
Lemma~\ref{lem:hmt_hsti_hmt}, Lemma~\ref{lem:hsti_sbt}, and
Lemma~\ref{lem:l1_norm}, we have that $\hat\gamma \geq \gamma(1-o(1))
> \gamma/2$ and $ \eb_a'\hSti\sbt = \eb_a'\Sti\sbt(1+o(1)) \leq
\frac{3}{2}\eb_a'\Sti\sbt$ with probability at least
$1-\Ocal\rbr{\log^{-1}(n)}$. Then
\begin{align*}
\tilde v_a 
&\geq \frac{\gamma}{2}(\beta_a + \delta_2) - \frac{3}{2}\lambda\eb_a'\Sti\sbt\\
& \geq
 \frac
{\pi_1\pi_2(1 + \lambda\norm{\betab_T}_1)(\beta_a-|\delta_2|)
 - 3\lambda_0\norm{\Sti\sbt}_\infty}
{2(1+\pi_1\pi_2\snbt)},
\end{align*}
so that $\sgn(\tilde v_a) = \sgn(\beta_a)$ if 
\begin{equation}
\label{eq:proof:lemma1:final}
\pi_1\pi_2(1 + \lambda\norm{\betab_T}_1)(\beta_a-|\delta_2|)
 - 3\lambda_0\norm{\Sti\sbt}_\infty > 0.
\end{equation}
Lemma~\ref{lem:deviation:beta_t} gives a bound on $|\delta_2|$, for
each fixed $a \in T$, as 
\begin{equation*}
\begin{aligned}
  |\delta_2| 
&\leq
  C_1 \sqrt{
      \rbr{\Sti}_{aa}
\rbr{1\vee\snbt}
      \frac{\log(s\log(n))}{n}
    }  + C_2
    |\beta_a|
    \sqrt{
      \frac{\log(s\log(n))}{n}
    }.
  \end{aligned}
\end{equation*}
Therefore assumption \eqref{eq:beta_min:lemma:1}, with $K_\beta$
sufficiently large, and a union bound over all $a \in T$ implies
\eqref{eq:proof:lemma1:final}. 

Lemma~\ref{lem:deviation:beta_t} gives $\sgn(\bt) = \sbt$ with
probability $1-\Ocal(\log^{-1}(n))$.
\end{proof}

\begin{proof}[Proof of Lemma~\ref{lem:dual_certificate}]
  Throughout the proof, we will work on the event $\Acal$ defined in
  \eqref{eq:acal}.  By construction, the vector $\hat\vb =
  (\tilde\vb_T',\zero')'$ satisfies the condition in
  \eqref{eq:KKT:1}. Therefore, to show that it is a solution to
  \eqref{eq:opt_problem1}, we need to show that it also satisfies
  \eqref{eq:KKT:2}.

To simplify notation, let
\begin{eqnarray*}
  \Cb = \Sbb + \frac{n_1n_2}{n(n-2)}\hat\mub\hat\mub',\ 
  \hat\gamma = 
    \frac{n_1n_2}{n(n-2)}
    \frac{1 + \lambda\norm{\hat\bt}_1} 
      {1+\frac{n_1n_2}{n(n-2)}\norm{\hat\bt}_{\hSt}^2}, \\
  \text{ and }
  \gamma =   
    \frac{\pi_1\pi_2(1 + \lambda\norm{\betab_T}_1)}
      {1+\pi_1\pi_2\snbt}.
\end{eqnarray*}
Recall that $\tilde\vb_T = \hat \gamma \hSti\hmt - \lambda \hSti\sbt$.

Let $\Ub \in \RR^{(n-2)\times p}$ be a matrix with each row $\ub_i
\iidsim \Ncal(0, \Sigmab)$ such that $(n-2)\Sbb = \Ub'\Ub$.
For $a \in N$, we have 
\begin{equation*}
  (n-2)\Sbb_{aT} 
  = (\Ub_T\Sti\Sigmab_{Ta} + \Ub_{a\cdot T})'\Ub_T
  = \Sigmab_{aT}\Sti\Ub_T'\Ub_T + \Ub_{a\cdot T}'\Ub_T
\end{equation*}
where $\Ub_{a\cdot T} \sim \Ncal\rbr{0, \nr \sigma_{a|T}\Ib_{n-2}}$ is independent
of $\Ub_T$, 
and
\begin{equation*}
  \hat\mu_{a} = \Sigmab_{aT}\Sti\hmt + \hat\mu_{a\cdot T}
\end{equation*}
where $\hat\mu_{a\cdot T} \sim \Ncal\rbr{0,  \frac{n}{n_1n_2}\sigma_{a|T}}$
is independent of $\hmt$. Therefore, 
\begin{align*}
\Cb_{aT} 
& = \Sbb_{aT} + \nr \hat\mu_{a}\hmt'  \\
& = \Sigmab_{aT}\Sti\hSt + (n-2)^{-1}\Ub_{a\cdot T}'\Ub_T  + \nr \hat\mu_{a} \hmt',
\end{align*}
\begin{align*}
\Cb_{aT} \tilde\vb_T 
& = \hat\gamma\Sigmab_{aT}\Sti\hmt + \hat\gamma\nr\rbr{\hsnbt}\hat\mu_{a}\\
&\qquad -\lambda\rbr{\Sigmab_{aT}\Sti\sbt+\nr\norm{\hat \betab_T}_1\cdot\hat\mu_{a}}\\
&\qquad +(n-2)^{-1}\Ub_{a\cdot T}'\Ub_T\tilde\vb_T \\
& = 
\rbr{
\hat\gamma +
\hat\gamma\nr\hsnbt
-\lambda\nr\norm{\hat \betab_T}_1
}
\Sigmab_{aT}\Sti\hmt\\
&\qquad -\lambda\Sigmab_{aT}\Sti\sbt \\
&\qquad +(n-2)^{-1}\Ub_{a\cdot T}'\Ub_T\tilde\vb_T 
+ \hat\gamma\nr\rbr{\hsnbt}\hat\mu_{a\cdot T}\\
&\qquad-\lambda\nr\norm{\hat \betab_T}_1\cdot\hat\mu_{a\cdot T} \\
& = 
\nr\Sigmab_{aT}\Sti\hmt -\lambda\Sigmab_{aT}\Sti\sbt \\
&\qquad +(n-2)^{-1}\Ub_{a\cdot T}'\Ub_T\tilde\vb_T 
+ \hat\gamma\nr\rbr{\hsnbt}\hat\mu_{a\cdot T}\\
&\qquad-\lambda\nr\norm{\hat \betab_T}_1\cdot\hat\mu_{a\cdot T},
\end{align*}
and finally 
\begin{align*}
\Cb_{aT} \tilde\vb_T - \nr\hat\mu_{a} 
& = 
-\lambda\Sigmab_{aT}\Sti\sbt +(n-2)^{-1}\Ub_{a\cdot T}'\Ub_T\tilde\vb_T \\
&\quad +
\nr\rbr{
\hat\gamma\hsnbt
-\lambda\norm{\hat \betab_T}_1
-1
}
\hat\mu_{a\cdot T}.
\end{align*}

First, we deal with the term 
\begin{align*}
  (n-2)^{-1}\Ub_{a\cdot T}'\Ub_T\tilde\vb_T = 
  \underbrace{
    \frac{\hat \gamma}{n-2}\Ub_{a\cdot T}'\Ub_T\hSti\hmt 
  }_{T_{1,a}}
  - 
  \underbrace{
    \frac{\lambda}{n-2}\Ub_{a\cdot T}'\Ub_T\hSti\sbt  
  }_{T_{2,a}}.
\end{align*}
Conditional on $\{y_i\}_{i\in[n]}$ and $\Xb_T$, we have that
\begin{equation*}
  T_{1,a} \sim \Ncal\rbr{0,
    \nr \sigma_{a|T} \frac{\hat\gamma^2}{n-2} \hsnbt
  }
\end{equation*}
and
\begin{equation*}
\max_{a\in N} |T_{1,a}| \leq
\sqrt{
2\nr \rbr{\max_{a\in N}\sigma_{a|T}} \frac{\hat\gamma^2}{n-2} \hsnbt
\log\rbr{(p-s)\log(n)}
}  
\end{equation*}
with probability at least $1-\log^{-1}(n)$. On the event $\Acal$, we have
that 
\begin{equation*}
\begin{aligned}
\max_{a\in N} |T_{1,a}| &\leq
(1+o(1))
\sqrt{2
\pi_1\pi_2 \gamma^2 \rbr{\max_{a\in N}\sigma_{a|T}} \snbt
\frac{\log\rbr{(p-s)\log(n)}}{n}
} \\
&=(1+o(1))\sqrt{2}
  \pi_1\pi_2(1 + \lambda\norm{\betab_T}_1)
  \frac{\lambda}{K_{\lambda_0}}.
\end{aligned}
\end{equation*}
Since
\begin{equation*}
  \norm{\betab_T}_1 
  \leq \sqrt{s}\norm{\betab_T}_2 
  =  \sqrt{s}\norm{\St^{-1/2}\St^{1/2}\betab_T}_2 
  \leq \sqrt{s\Lambda_{\min}^{-1}(\St)\snbt}
\end{equation*}
and 
\begin{equation*}
  \begin{aligned}
  \lambda\norm{\betab_T}_1 
  &=  \frac{\lambda_0\norm{\betab_T}_1}{1+\pi_1\pi_2\snbt} 
  \leq \frac{\lambda_0\sqrt{s\Lambda_{\min}^{-1}(\St)\snbt}}{1+\pi_1\pi_2\snbt}\\
&\leq
  \frac{K_{\lambda_0}}{\sqrt{K}}
  \frac{\sqrt{\rbr{1\vee\snbt}\snbt}}{1+\pi_1\pi_2\snbt}
\leq   \frac{K_{\lambda_0}}{\pi_1\pi_2\sqrt{K}},
\end{aligned}
\end{equation*}
we have that 
\begin{equation*}
\begin{aligned}
\max_{a\in N} |T_{1,a}| &\leq
(1+o(1))\sqrt{2}
  \pi_1\pi_2\rbr{K_{\lambda_0}^{-1} + \rbr{\pi_1\pi_2\sqrt{K}}^{-1}}
  \lambda
  < (\alpha/3) \lambda
\end{aligned}
\end{equation*}
by taking both $K_{\lambda_0}$ and $K$ sufficiently large.

Similarly, conditional on $\{y_i\}_{i\in[n]}$ and $\Xb_T$, we have that
\begin{equation*}
  T_{2,a} \sim \Ncal\rbr{0,
    \nr \sigma_{a|T} \frac{\lambda^2}{n-2} \sbt'\hSti\sbt
  },
\end{equation*}
which, on the event $\Acal$, gives 
\begin{equation*}
\begin{aligned}
\max_{a\in N} |T_{2,a}| & \leq
(1+o(1))\lambda
\sqrt{2
\pi_1\pi_2 \rbr{\max_{a\in N}\sigma_{a|T}}\Lambda_{\min}^{-1}(\St)
 \frac{s\log\rbr{(p-s)\log(n)}}{n}
}\\ 
& \leq 
(1+o(1))
\sqrt{\frac{2}{K}}\lambda < (\alpha/3)\lambda
\end{aligned}
\end{equation*}
with probability at least $1-\log^{-1}(n)$ for sufficiently large $K$. 

Next, let 
\begin{equation*}
  T_{3,a} = \nr\rbr{
\hat\gamma\hsnbt
-\lambda\norm{\hat \betab_T}_1
-1
}
\hat\mu_{a\cdot T}.
\end{equation*}
Simple algebra shows that 
\[
\hat\gamma\hsnbt
-\lambda\norm{\hat \betab_T}_1
-1 = -\frac{1+\lambda\norm{\hat \betab_T}_1}{1+\nr\hsnbt}.
\]
Therefore conditional on $\{y_i\}_{i\in[n]}$ and $\Xb_T$, we have
that
\begin{equation*}
  T_{3,a} \sim
  \Ncal\rbr{0,
    \rbr{\nr
      \frac{1+\lambda\norm{\hat \betab_T}_1}{1+\nr\hsnbt}     
    }^2 \frac{n}{n_1n_2}\sigma_{a|T}
  },
\end{equation*}
which, on the event $\Acal$, gives
\begin{equation*}
\begin{aligned}
\max_{a\in N} |T_{3,a}| 
&\leq
(1+o(1))
\frac{1+\lambda\norm{\betab_T}_1}{1+\pi_1\pi_2\snbt}
\sqrt{2\pi_1\pi_2
\rbr{\max_{a\in N}\sigma_{a|T}}
 \frac{\log\rbr{(p-s)\log(n)}}{n} 
} \\
&\leq
(1+o(1))\sqrt{2}
\frac{1+\lambda\norm{\betab_T}_1}{1+\pi_1\pi_2\snbt}
\frac{1}{K_{\lambda_0}\sqrt{1\vee\snbt}}
\lambda < (\alpha/3)\lambda
\end{aligned}
\end{equation*}
with probability at least $1-\log^{-1}(n)$ when $K_{\lambda_0}$ and
$K$ are chosen sufficiently large. 

Piecing all these results together, we have that 
\[
\max_{a\in N}|\Cb_{aT}
\tilde\vb_T - \nr\hat\mu_{a}| < 1.
\]
\end{proof}

\subsection{Proof of Theorem~\ref{thm:lower_bound}}

The theorem will be shown using standard tools described in
\cite{tsybakov2009}. First, in order to provide a lower bound on the
minimax risk, we will construct a finite subset of $\Theta(\Sigmab,
\tau, s)$, which contains the most difficult instances of the
estimation problem so that estimation over the subset is as difficult
as estimation over the whole family. Let $\Theta_1 \subset
\Theta(\Sigmab, \tau, s)$, be a set with finite number of elements, so
that
  \[
    \inf_{\Psi} R(\Psi, \Theta(\Sigmab, \tau, s) \geq
    \inf_{\Psi} \max_{\thetab \in \Theta_1}
    \PP_{\thetab}[\Psi(\{\xb_i, y_i\}_{i\in[n]}) \neq T(\thetab)].
  \]
  To further lower bound the right hand side of the display above, we
  will use Theorem 2.5 in \cite{tsybakov2009}. Suppose that $\Theta_1 =
  \{ \thetab_0, \thetab_1, \ldots, \thetab_M \}$ where $T(\thetab_a)
  \neq T(\thetab_b)$ and 
  \begin{equation}
    \label{eq:lower_bound_kl_requirement}
    \frac{1}{M} \sum_{a=1}^M KL(\PP_{\thetab_0}|\PP_{\thetab_a}) 
    \leq \alpha \log(M),\quad \alpha \in (0, 1/8) 
  \end{equation}
  then 
  \[
    \inf_{\Psi} R(\Psi, \Theta(\beta, s) \geq
    \frac{\sqrt{M}}{1+\sqrt{M}}
    \left(1 - 2\alpha - \sqrt{\frac{2\alpha}{\log(M)}}\right).
  \]
  Without loss of generality, we will consider $\thetab_a = (\mub_a,
  \zero, \Sigmab)$. Denote $\PP_{\thetab_a}$ the joint distributions of
  $\{\Xb_i, Y_i\}_{i\in[n]}$. Under $\PP_{\thetab_a}$, we have
  $\PP_{\thetab_a}(Y_i=1)=\PP_{\thetab_a}(Y_i=2)=\frac{1}{2}$,
  $\Xb_i|Y_i=1 \sim \Ncal(\zero, \Sigmab)$ and $\Xb_i|Y_i=2 \sim
  \Ncal(\mub_a, \Sigmab)$. Denote $f(\xb; \mub, \Sigmab)$ the density
  function of a multivariate Normal distribution. With this we have
  \begin{equation}
    \label{eq:proof:kl_distance}
    \begin{aligned}
      KL(\PP_{\thetab_0}|\PP_{\thetab_a}) & =
      \EE_{\thetab_0}\log\frac{d\PP_{\thetab_0}}{d\PP_{\thetab_a}} \\
      &=
      \EE_{\thetab_0}\log
      \frac{\prod_{i\in[n]}d\PP_{\thetab_0}[\Xb_i|Y_i]\PP_{\thetab_0}[Y_i]}
      {\prod_{i\in[n]}d\PP_{\thetab_a}[\Xb_i|Y_i]\PP_{\thetab_a}[Y_i]}\\
      &=
      \EE_{\thetab_0}\sum_{i\ :\ y_i=2}
      \log
      \frac{ f(\Xb_i; \mub_0, \Sigmab)}
      {f(\Xb_i; \mub_a, \Sigmab)} \\
      &= \frac{\EE_{\thetab_0}n_2}{2} (\mub_0-\mub_a)'\Sigmab^{-1}(\mub_0-\mub_a) \\
      &= \frac{n}{4} (\betab_0-\betab_a)'\Sigmab(\betab_0-\betab_a) 
    \end{aligned}
  \end{equation}
  where $\betab_a = \Sigmab^{-1}\mub_a$.  We proceed to construct
  different finite collections for which
  \eqref{eq:lower_bound_kl_requirement} holds.

  Consider a collection $\Theta_1 = \{ \thetab_0, \thetab_1, \ldots,
  \thetab_{p-s} \}$, with $\thetab_a = (\mub_{a}, \zero)$, that
  contains instances whose supports differ in only one component.
  Vectors $\{\mub_a\}_{a=0}^{p-s}$ are constructed indirectly through
  $\{\betab_a\}_{a=0}^{p-s}$, using the relationship $\betab_a =
  \Sigmab^{-1}\mub_a$. Note that this construction is possible, since
  $\Sigmab$ is a full rank matrix. For every $a$, all $s$ non-zero
  elements of the vector $\betab_a$ are equal to $\tau$. Let $T$ be
  the support and $u(T)$ an element of the support $T$ for which
  \eqref{eq:close_ensemble:phi} is minimized. Set $\betab_0$ so
  that $\supp(\betab_0) = T$. The remaining $p-s$ parameter vectors
  $\{\betab_a\}_{a=1}^{p-s}$ are constructed so that the support of
  $\betab_a$ contains all $s-1$ element in $T \bks u(T)$ and then one
  more element from $[p]\bks T$. With this,
  \eqref{eq:proof:kl_distance} gives
  \begin{equation*}
      KL(\PP_{\thetab_0}|\PP_{\thetab_a}) 
      = \frac{n\tau^2}{4} (\Sigma_{uu}+\Sigma_{vv} - 2\Sigma_{uv})
  \end{equation*}
  and \eqref{eq:close_ensemble:phi} gives
  \[
  \frac{1}{p-s}\sum_{a=1}^{p-s}KL(\PP_{\thetab_0}|\PP_{\thetab_a}) 
  = \frac{n\tau^2}{4}\varphi_{\rm close}(\Sigmab).
  \]
  It follows from the display above that if 
  \begin{equation}
    \label{eq:beta_min_ensemble}
    \tau <
    \sqrt{\frac{4}{\varphi_{\rm close}(\Sigmab)}
    \frac{\log(p-s+1)}{n}},
  \end{equation}
  then \eqref{eq:lower_bound_kl_requirement} holds with $\alpha =
  1/16$.

  Next, we consider another collection $\Theta_2 = \{\thetab_0,
  \thetab_1, \ldots, \thetab_M\}$, where $M = {p-s \choose s}$, and
  the Hamming distance between $T(\thetab_0)$ and $T(\thetab_a)$ is
  equal to $2s$. As before, $\thetab_a = (\mub_a, \zero)$ and vectors
  $\{\mub_a\}_{a=0}^M$ are constructed so that $\betab_a =
  \Sigmab^{-1}\mub_a$ with $s$ non-zero components equal to $\tau$.
  Let $T$ be the support set for which the minimum in
  \eqref{eq:far_ensemble:phi} is attained. Set the vector
  $\betab_0$ so that $\supp(\betab_0) = T$. The remaining vectors
  $\{\betab_a\}_{a=1}^M$ are set so that their support contains $s$
  elements from the set $[p]\bks T$. Now,
  \eqref{eq:proof:kl_distance} gives
  \begin{equation*}
    \begin{aligned}
      KL(\PP_{\thetab_0}|\PP_{\thetab_a}) 
      = \frac{n\tau^2}{4} \one'
      \Sigmab_{T(\thetab_0)\cup T(\thetab_a),T(\thetab_0)\cup T(\thetab_a)}\one.
    \end{aligned}
  \end{equation*}
  Using \eqref{eq:far_ensemble:phi}, if 
  \begin{equation}
    \label{eq:beta_min_ensemble_2}  
    \tau < \sqrt{
      \frac{4}{\varphi_{\rm far}(\Sigmab)}
      \frac{\log{p-s \choose s}}{n}},
  \end{equation} 
  then \eqref{eq:lower_bound_kl_requirement} holds with $\alpha = 1/16$.

  Combining \eqref{eq:beta_min_ensemble} and
  \eqref{eq:beta_min_ensemble_2}, by taking the larger $\beta$
  between the two, we obtain the result.

\subsection{Proof of Theorem~\ref{thm:exhaustive_search}}

For a fixed $T$, let $\Delta(T') = f(T) - f(T')$ and $\Tcal = \{ T'
  \subset [p]\ :\ |T'|=s, T'\neq T\}$.  Then
\begin{equation*}
\PP_{T}[\hat T \neq T] = \PP_{T}[\bigcup_{T' \in \Tcal} \{\Delta(T') < 0\}]
\leq \sum_{T' \in \Tcal}\PP_{T}[\Delta(T') < 0].
\end{equation*}

Partition $\hmt = (\hmp{1}', \hmp{2}')'$, where $\hmp{1}$ contains the
variables in $T\cap T'$, and $\hmp{T'} = (\hmp{1}',
\hmp{3}')'$. Similarly, we can partition the covariance matrix $\hSt$
and $\hSp{T'T'}$. Then,
\[
g(T) = \hmp{1}'\hSip{11}\hmp{1} +
\tilde\mub_{2\mid1}'\hSip{22\mid1}\tilde\mub_{2\mid1}
\]
where $\tilde\mub_{2\mid1} = \hmp{2} - \hSp{21}\hSip{11}\hmp{1}$ and 
$\hSp{22\mid1} = \hSp{22} - \hSp{21}\hSip{11}\hSp{12}$ \citep[see
Section 3.6.2 in][]{mardia1980multivariate}. Furthermore, we have that 
\begin{equation}
  \label{eq:delta_diff_form}
\Delta(T') =
\tilde\mub_{2\mid1}'\hSip{22\mid1}\tilde\mub_{2\mid1} - 
\tilde\mub_{3\mid1}'\hSip{33\mid1}\tilde\mub_{3\mid1}.  
\end{equation}

The two terms are correlated, but we will ignore this correlation and
use the union bound to lower bound the first term and upper bound the
second term. We start with analyzing
$\tilde\mub_{2\mid1}'\hSip{22\mid1}\tilde\mub_{2\mid1}$, noting
that the result for the second term will follow in the same way. By
Theorem 3.4.5 in \citet{mardia1980multivariate}, we have that
\[
\hSp{22\mid1} \sim \Wcal_{s - |T \cap T'|}
\rbr{(n-2)^{-1}\Sigmab_{22\mid 1}, n - 2 - |T \cap T'|}
\]
and independent of $(\Sbb_{12}, \Sbb_{11}, \hat\mub)$. Therefore
$\hSp{22\mid1}$ is independent of $\tilde\mub_{2\mid1}$ and
Theorem 3.2.12 in \cite{muirhead1982aspects} gives us that 
\[
(n-2)\frac{\tilde\mub_{2\mid1}'\Sigmab_{22\mid1}^{-1}\tilde\mub_{2\mid1}}
{\tilde\mub_{2\mid1}'\hSip{22\mid1}\tilde\mub_{2\mid1}}
\sim \chi^2_{n-1-s}.
\]
As in Lemma~\ref{lem:hmt_hsti_hmt_ratio}, we can show that 
\[
1 - C_1 \sqrt{\frac{\log(\eta^{-1})}{n}}
\leq
\frac
{\tilde\mub_{2\mid1}'\hSip{22\mid1}\tilde\mub_{2\mid1}}
{\tilde\mub_{2\mid1}'\Sigmab_{22\mid1}^{-1}\tilde\mub_{2\mid1}}
\leq
1 + C_2 \sqrt{\frac{\log(\eta^{-1})}{n}}.
\]
For
$\tilde\mub_{2\mid1}$, we have
\begin{equation*}
  \begin{aligned}
\tilde\mub_{2\mid1} =     
\hmp{2} - \hSp{21}\hSip{11}\hmp{1}
=
\hmp{2\mid1}+
\Sigmab_{21}\Sigmab^{-1}_{11}\hmp{1}
 - \hSp{21}\hSip{11}\hmp{1}
  \end{aligned},
\end{equation*}
where $\hmp{2\mid1} \sim \Ncal(\mub_{2\mid1},
\frac{n}{n_1n_2}\Sigmab_{22\mid1})$, independent of $\hmp{1}$, and
$\mub_{2\mid1} = \mub_2 -
\Sigmab_{21}\Sigmab_{11}^{-1}\mub_1$. Conditioning on $\hmp{1}$ and
$\hSp{1}$, we have that
\[
\hSp{21}\hSip{11}\hmp{1} = \Sigmab_{21}\Sigmab_{11}^{-1}\hmp{1} +
\Zb,
\] 
where
$\Zb\sim\Ncal\rbr{\zero,(n-2)^{-1}\hmp{1}'\hSip{11}\hmp{1}\Sigmab_{22\mid1}}$.
Since $\hmp{2\mid1}$ is independent of $(\Sbb_{12},
\Sbb_{11},\hmp{1})$, we have that
\[
\tilde\mub_{2\mid1} | \hmp{1}, \Sbb_{11}
\sim
\Ncal\rbr{
\mub_{2\mid1}, 
\rbr{\frac{n}{n_1n_2}+(n-2)^{-1}\hmp{1}'\hSip{11}\hmp{1}}\Sigmab_{22\mid1}
}.
\]
Let $a = \frac{n}{n_1n_2}+(n-2)^{-1}\hmp{1}'\hSip{11}\hmp{1}$. Then
\[
\tilde\mub_{2\mid1}'
\Sigmab_{22\mid1}^{-1}
\tilde\mub_{2\mid1}
\mid \hmp{1}, \Sbb_{11} \sim
a\chi^2_{|T\bks T'|}\rbr{
a^{-1}\mub_{2\mid1}'
\Sigmab_{22\mid1}^{-1}
\mub_{2\mid1}
}.
\]
Therefore, conditioned on $(\hmp{1}, \hSp{11})$,
\[
\begin{aligned}
&\tilde\mub_{2\mid1}'\hSip{22\mid1}\tilde\mub_{2\mid1}\\
&\geq
\rbr{1 - C_1 \sqrt{\frac{\log(\eta^{-1})}{n}}} \\
&\qquad \times
\rbr{
\rbr{\mub_{2\mid1}'\Sigmab_{22\mid1}^{-1}\mub_{2\mid1} + a|T\bks T'|}
-2\sqrt{
\rbr{2a\mub_{2\mid1}'\Sigmab_{22\mid1}^{-1}\mub_{2\mid1} +
  a^2|T\bks T'|}
\log(\eta^{-1})
}
}
\end{aligned}
\]
with probability $1-2\eta$. 
Similarly, 
\[
\begin{aligned}
&\tilde\mub_{3\mid1}'\hSip{33\mid1}\tilde\mub_{3\mid1}\\
&\leq
\rbr{1 + C_2 \sqrt{\frac{\log(\eta^{-1})}{n}}} \\
&\qquad \times
\rbr{
\rbr{\mub_{3\mid1}'\Sigmab_{33\mid1}^{-1}\mub_{3\mid1} + a|T'\bks T|}
+2\sqrt{
\rbr{2a\mub_{3\mid1}'\Sigmab_{33\mid1}^{-1}\mub_{3\mid1} +
  a^2|T'\bks T|}
\log(\eta^{-1})
} 
+2a\log(\eta^{-1})
}
\end{aligned}
\]
with probability $1-2\eta$. Finally, Lemma~\ref{lem:hmt_hsti_hmt}
gives that $|a| \leq C\rbr{1 \vee \mub_{
    1}'\Sigmab_{11}^{-1}\mub_{1}}n^{-1}$ with probability $1-2\eta$.

Set $\eta_k = \rbr{{p-s \choose s-k}{s \choose k}s\log(n)}^{-1}$.  For
any $T' \subset [p]$, where $|T'| = s$ and $|T' \cap T| = k$, we have
that
\begin{equation*}
  \begin{aligned}
\tilde\mub_{2\mid1}'\hSip{22\mid1}\tilde\mub_{2\mid1} -
\tilde\mub_{3\mid1}'\hSip{33\mid1}\tilde\mub_{3\mid1} 
&\geq
(1-o(1))\mub_{2\mid1}'\Sigmab_{22\mid1}^{-1}\mub_{2\mid1}
-
(1+o(1))\mub_{3\mid1}'\Sigmab_{33\mid1}^{-1}\mub_{3\mid1}\\
&\ \ -C\sqrt{
\rbr{1 \vee \mub_{ 1}'\Sigmab_{11}^{-1}\mub_{1}}
  \mub_{ 2\mid1}'\Sigmab_{22\mid1}^{-1}\mub_{2\mid1}
\Gamma_{n,p,s,k}
}
\\
&\ \ -
C\rbr{1 \vee \mub_{ 1}'\Sigmab_{11}^{-1}\mub_{1}}\Gamma_{n,p,s,k}
.
  \end{aligned}
\end{equation*} The right hand side in the above display is
bounded away from zero with probability $1-\Ocal\rbr{\rbr{{p-s \choose
      s-k}{s \choose k}s\log(n)}^{-1}}$ under the assumptions.
Therefore,
\[
\PP_{T}[\hat T \neq T] 
\leq \sum_{k = 0}^{s-1}\ \sum_{T' \in \Tcal\ :\ |T \cap T'| = k}\PP_{T}[\Delta(T') < 0]
 \leq \frac{C}{\log(n)},
\]
which completes the proof.

\section{Proof of Risk Consistency}

In this section, we give a proof of Corollary~\ref{thm:risk}.  From
Theorem~\ref{thm:main:sda} we have that $\hat\vb =
(\hat\vb_T',\zero')'$ with $\hat\vb_T$ defined in
\eqref{eq:solution:hat_wt}.  Define
\[
 \tilde\vb_T = 
 \frac{n(n-2)}{n_1n_2}
 \rbr{
  1+\frac{n_1n_2}{n(n-2)}\hmt'\hSti\hmt } \hat\vb_T.
\]
To obtain a bound on the risk, we need to control 
\begin{equation}
\label{eq:proof:risk:1}
\frac{-\tilde\vb_T'(\mub_{i,T} -
    \hat\mub_{i,T})-\tilde\vb_T'\hmt/2}
{\sqrt{\tilde\vb_T'\St\tilde\vb_T}}
\end{equation}
for $i\in\{1,2\}$.
Define the following quantities
\begin{gather*}
\delta_1=\hmt'\hSti\sbt-\norm{\bt}_1,\quad
\tilde\delta_1 = \delta_1/\norm{\bt}_1, \\
\delta_2=\hsnbt-\snbt, \quad \text{and}\quad
\tilde\delta_2 = \delta_2/\snbt.
\end{gather*}
Under the assumptions, we have that
\begin{align*}
\lambda_0 &=\Ocal\rbr{\sqrt{\frac{\Lambda_{\min}(\St)\snbt}{K(n)s}}},
\\
r_n 
& =
\Ocal\rbr{\frac{\lambda_0\norm{\bt}_1}{\snbt}}
=
 \Ocal\rbr{
\frac{\norm{\bt}_1}{\nbt}
\sqrt{\frac{\Lambda_{\min}(\St)}{K(n)s}}
}, \text{ and} \\
\tilde\delta_2 &= 
\Ocal_P\rbr{\sqrt{\frac{\log\log(n)}{n}} \vee
  \frac{s\vee\log\log(n)}{\snbt n}}.
\end{align*}
The last equation follows from Lemma~\ref{lem:hmt_hsti_hmt}.
Note that $\tilde\delta_2 = \Ocal(r_n)$.
From Lemma~\ref{lem:l1_norm}, we have that $\tilde\delta_1 = o_p(1)$.

We have $\lambda\hmt'\hSti\sbt =
\lambda\norm{\bt}_1(1+\Ocal_P(\tilde\delta_1)) = \Ocal_P(r_n)$, since
Lemma~\ref{lem:l1_norm} gives $\tilde\delta_1 = o_p(1)$, and 
\[
\frac{ \frac{n(n-2)}{n_1n_2} + \hmt'\hSti\hmt }
        { 1+\pi_1\pi_2\snbt} = \Ocal_P(1).
\]
Therefore $\tilde\vb_T = (1+\Ocal_P(r_n))\hSti\hmt -
\Ocal_P(1)\lambda_0\hSti\sbt$. With this, we have
\begin{equation}
  \begin{aligned}
\label{eq:proof:risk:2}
\tilde\vb_T'\hmt 
& =
(1+\Ocal_P(r_n))(1+\Ocal_P(\tilde\delta_2))\snbt  - \Ocal_P(1)
(1+\Ocal_P(\tilde\delta_1))
\lambda_0 \norm{\bt}_1\\
&= 
\snbt\sbr{
1 + \Ocal_P\rbr{r_n}
- \Ocal_P(1)\frac{\lambda_0\norm{\bt}_1}{\snbt}
} \\
&= 
\snbt\sbr{
1 + \Ocal_P\rbr{r_n}
},
  \end{aligned}
\end{equation}
where the last line follows from $\lambda\norm{\bt}_1 \asymp
\lambda_0\norm{\bt}_1/\snbt$. Next
\[
\begin{aligned}
(\mub_{1,T} &- \hat\mub_{1,T})'\hSti\hmt \\
& \leq 
\norm{\hSt^{-1/2}(\mub_{1,T} - \hat\mub_{1,T})}_2
\norm{\hSt^{-1/2}\hmt}_2 \\
& \leq (1+\Ocal_P(\sqrt{s/n}))\Lambda^{-1/2}_{\min}(\St)\sqrt{s}
\norm{\mub_{1,T} -
  \hat\mub_{1,T}}_\infty\nbt\sqrt{1+\Ocal_P(\tilde\delta_2)}\\
&= \nbt\Ocal_P\rbr{\sqrt{\Lambda^{-1}_{\min}(\St)s\log\log(n)/n}}
\end{aligned}
\]
and similarly
\[
(\mub_{1,T} - \hat\mub_{1,T})'\hSti\sbt = 
\sqrt{q_n}\Ocal_P\rbr{\sqrt{\Lambda_{\min}^{-1}(\St)s\log\log(n)/n}}.
\]
Combining these two estimates, we have
\begin{eqnarray}
\label{eq:proof:risk:3}\\
\lefteqn{\abr{\tilde\vb_T'(\mub_{1,T} - \hat\mub_{1,T})} }\nonumber \\
&=&\nbt\rbr{1-\lambda_0\sqrt{q_n}/\nbt} \Ocal_P\rbr{\sqrt{\Lambda_{\min}^{-1}(\St)s\log\log(n)/n}} \nonumber\\ 
&=&\nbt 
\Ocal_P\rbr{\sqrt{\Lambda_{\min}^{-1}(\St)s\log\log(n)/n}} . \nonumber
\end{eqnarray}
From \eqref{eq:proof:risk:2} and \eqref{eq:proof:risk:3}, we have that 
\begin{equation}
\label{eq:proof:risk:3.5}
  \begin{aligned}
-\rbr{\tilde\vb_T'(\mub_{1,T} - \hat\mub_{1,T})} - \tilde\vb_T'\hmt/2
= -\frac{\snbt}{2}\rbr{1+\Ocal_P(r_n)}.
  \end{aligned}
\end{equation}
Finally, a simple calculation gives, 
\begin{equation}
\label{eq:proof:risk:4}
  \begin{aligned}
&\tilde\vb_T'\St\tilde\vb_T  \\
&\leq 
\Lambda_{\max}\rbr{\hSt^{-1/2}\St\hSt^{-1/2}}\\
&\quad\times
\Big(
\rbr{1+\Ocal_P(r_n)}^2\hsnbt + \Ocal_P(1)\lambda_0^2\sbt'\hSti\sbt\\
&\qquad\qquad -\Ocal_P(1)\lambda_0\hmt'\hSti\sbt
\Big) \\
&= \snbt
\rbr{1+\Ocal_P\rbr{r_n\vee\tilde\delta_2
\vee\frac{\lambda_0^2q_n}{\snbt}\vee\sqrt{\frac{s}{n}}}} \\
&= \snbt
\rbr{1+\Ocal_P\rbr{r_n
\vee\frac{\lambda_0^2q_n}{\snbt}}}.
  \end{aligned}
\end{equation}
Combining the equation \eqref{eq:proof:risk:3.5} and
\eqref{eq:proof:risk:4}, we have that
\[
\begin{aligned}
&
\frac{-\tilde\vb_T'(\mub_{1,T} -
    \hat\mub_{1,T})-\tilde\vb_T'\hmt/2}
{\sqrt{\tilde\vb_T'\St\tilde\vb_T}} =
-\frac{\nbt}{2}
\frac{\rbr{1+\Ocal_P\rbr{r_n}}}
{\sqrt{1+\Ocal_P\rbr{r_n
\vee\frac{\lambda_0^2q_n}{\snbt}}}}.
\end{aligned}
\]
This completes the proof.

\section{Technical Results}

We provide some technical lemmas which are useful for proving the main
results. Without loss of generality, $\pi_1 = \pi_2 = 1/2$ in model
\eqref{eq:model}. Define
\begin{equation}
  \label{eq:event_n}
  \Ecal_n = \Bigl\{\frac{n}{4} \leq n_1 \leq \frac{3n}{4} \Bigr\} 
            \cap
            \Bigl\{\frac{n}{4} \leq n_2 \leq \frac{3n}{4} \Bigr\},
\end{equation}
where $n_1, n_2$ are defined in $\S$\ref{sec:introduction}.
Observe that $n_1 \sim {\rm Binomial}(n, 1/2)$, which gives $\PP[\{n_1
\leq n/4 \}] \leq \exp(-3n/64)$ and $\PP[\{n_1 \geq 3n/4 \}] \leq
\exp(-3n/64)$ using standard tail bound for binomial random variable
\cite[p.~130]{devroye1996probabilistic}. Therefore
\begin{equation}
  \label{eq:event_n_prob}
  \PP[\Ecal_n] \geq 1 - 4\exp(-3n/64).
\end{equation}
The analysis is performed by conditioning on $\yb$ and, in particular,
we will perform analysis on the event $\Ecal_n$. Note that on
$\Ecal_n$, $16/9 n^{-1} \leq n/(n_1n_2) \leq 16n^{-1}$. In our
analysis, we do not strive to obtain the sharpest possible
constants. 

\subsection{Deviation of the Quadratic Scaling Term}

In this section, we collect lemmas that will help us deal with
bounding the deviation of $\hmt'\hSti\hmt$ from $\mt'\Sti\mt$.

\begin{lemma}
\label{lem:hmt_hsti_hmt_ratio}
Define the event
\begin{equation}
  \label{eq:event_1}
  \Ecal_1(\eta) = \cbr{
1 - C_1\sqrt{\frac{\log(\eta^{-1})}{n}} \leq 
\frac{\hmt'\hSti\hmt}{\hmt'\Sti\hmt}
\leq 1 + C_2\sqrt{\frac{\log(\eta^{-1})}{n}}
  },
\end{equation}
for some constants $C_1,C_2 > 0$. Assume that $s=o(n)$, then
$\PP[\Ecal_1(\eta)]\geq 1-\eta$ for $n$ sufficiently large.
\end{lemma}
\begin{proof}[Proof of Lemma~\ref{lem:hmt_hsti_hmt_ratio}]
Using Theorem 3.2.12 in \citet{muirhead1982aspects}
\begin{equation*}
(n-2)\frac{\hmt'\Sti\hmt}{\hmt'\hSti\hmt}
\sim
\chi^2_{n-1-s}
\end{equation*}
\eqref{eq:chi-central-upper-johnstone} gives
\begin{equation*}
\frac{n-2}{n-1-s}
\frac{1}{1+\sqrt{\frac{16\log(\eta^{-1})}{3(n-1-s)}}} \leq
\frac{\hmt'\hSti\hmt}{\hmt'\Sti\hmt}
 \leq \frac{n-2}{n-1-s} \frac{1}{1-\sqrt{\frac{16\log(\eta^{-1})}{3(n-1-s)}}}
\end{equation*}
with probability at least $1 - \eta$. Since $s = o(n)$, the above
display becomes
\begin{equation}
  \label{eq:proof:hmt_hsti_hmt:t1}
1 - C_1\sqrt{\frac{\log(\eta^{-1})}{n}} \leq 
\frac{\hmt'\hSti\hmt}{\hmt'\Sti\hmt}
\leq 1 + C_2\sqrt{\frac{\log(\eta^{-1})}{n}}
\end{equation}  
for $n$ sufficiently large.
\end{proof}

\begin{lemma}
\label{lem:hmt_sti_hmt} 
Define the event
\begin{equation}
  \label{eq:event_2}  
\begin{aligned}
\Ecal_2(\eta) =
&\cbr{
  \hmt'\Sti\hmt \leq \snbt + 
   C_1 \left(
     \frac{s \vee \log(\eta^{-1})}{n} \vee
     \sqrt{\frac{\snbt\log(\eta^{-1})}{n}}
   \right)
} \\    
&\bigcap
\cbr{
  \hmt'\Sti\hmt \geq \snbt -
   C_2 \left(
     \frac{s}{n} \vee
     \sqrt{\frac{\snbt\log(\eta^{-1})}{n}}
   \right)
}.
\end{aligned}
\end{equation}
Assume that $\beta_{\min} \geq cn^{-1/2}$, then
$\PP[\Ecal_2(\eta)]\geq 1-2\eta$ for $n$ sufficiently large.
\end{lemma}

\begin{proof}[Proof of Lemma~\ref{lem:hmt_sti_hmt}]
Recall that $\hmt \sim \Ncal(\mt, \frac{n}{n_1n_2}\St)$.
Therefore
\begin{equation*}
  \frac{n_1n_2}{n}\hmt'\Sti\hmt
  \sim \chi^2_{s}\left(\frac{n_1n_2}{n}\mt'\Sti\mt\right).
\end{equation*}

Using \eqref{eq:chi-noncentral-upper}, we have that
\begin{equation}
  \label{eq:proof:hmt_hsti_hmt:t2:upper}
\begin{aligned}   
  \hmt'&\Sti\hmt \\ & \leq \snbt + 
   \frac{ns}{n_1n_2} +
  \frac{2n}{n_1n_2} \sqrt{\left(s +
    2\frac{n_1n_2}{n}\snbt\right)\log(\eta^{-1})}
    + \frac{2n}{n_1n_2}\log(\eta^{-1}) \\
   &\leq \snbt + 
   \frac{16 s}{n} +
    32 \sqrt{\left(\frac{s}{n^2} +
    \frac{2\snbt}{n}\right)\log(\eta^{-1})}
    + \frac{32\log(\eta^{-1})}{n} \\
   &\leq \snbt + 
   C_1 \left(
     \frac{s}{n} \vee
     \sqrt{\frac{\snbt\log(\eta^{-1})}{n}} \vee
     \frac{\log(\eta^{-1})}{n} 
   \right),
\end{aligned}
\end{equation}
with probability $1-\eta$. The second inequality follows since we are
working on the event $\Ecal_n$, and the third inequality follows from
the fact that $\beta_{\min}\geq cn^{-1/2}$.  A lower bound follows
from \eqref{eq:chi-noncentral-lower},
\begin{equation}
  \label{eq:proof:hmt_hsti_hmt:t2:lower}
\begin{aligned}   
  \hmt'&\Sti\hmt \\
  & \geq \snbt +
   \frac{ns}{n_1n_2} -
  \frac{2n}{n_1n_2} \sqrt{\left(s +
    2\frac{n_1n_2}{n}\snbt\right)\log(\eta^{-1})} \\
   &\geq \snbt + 
   \frac{16 s}{9n} -
    32 \sqrt{\left(\frac{s}{n^2} +
    \frac{2\snbt}{n}\right)\log(\eta^{-1})} \\
   &\geq \snbt - C_2\left(
     \frac{s}{n} \vee
     \sqrt{\frac{\snbt\log(\eta^{-1})}{n}}
     \right)
\end{aligned}
\end{equation}
with probability $1-\eta$.   
\end{proof}

\begin{lemma}
\label{lem:hmt_hsti_hmt}
On the event
$\Ecal_1(\eta)\cap\Ecal_2(\eta)$ the following holds
\begin{align*}
\abr{\hmt' \hSti\hmt - \snbt} \leq
    C \left(
     \rbr{\snbt\vee\nbt}
       \sqrt{\frac{\log(\eta^{-1})}{n}} \vee
     \frac{s \vee \log(\eta^{-1})}{n} 
   \right).
\end{align*}
\end{lemma}

\begin{proof}[Proof of Lemma~\ref{lem:hmt_hsti_hmt}]
On the event  $\Ecal_1(\eta)\cap\Ecal_2(\eta)$,
using Lemma~\ref{lem:hmt_hsti_hmt_ratio} and
Lemma~\ref{lem:hmt_sti_hmt}, we have that
\begin{equation*}
\begin{aligned}
&\hmt'\hSti\hmt \\
  &= \frac{\hmt'\hSti\hmt}{\hmt'\Sti\hmt}\hmt'\Sti\hmt\\
  &= \frac{\hmt'\hSti\hmt}{\hmt'\Sti\hmt}\mt'\Sti\mt + 
     \frac{\hmt'\hSti\hmt}{\hmt'\Sti\hmt}
       \rbr{\hmt'\Sti\hmt-\mt'\Sti\mt} \\
  &\leq \snbt + C_1\snbt\sqrt{\frac{\log(\eta^{-1})}{n}} +
    C_2 \left(
     \frac{s \vee \log(\eta^{-1})}{n} \vee
     \sqrt{\frac{\snbt\log(\eta^{-1})}{n}}
   \right).
\end{aligned}
\end{equation*}
A lower bound is obtained in the same way.
\end{proof}

\subsection{Other Results}

Let the event $\Ecal_3(\eta)$ be defined as
\begin{equation}
\label{eq:event_e3}
\Ecal_3(\eta) = \bigcap_{a \in [s]}
\left\{
\left| \Sti\hmt - \Sti\mt \right|
\leq
\sqrt{32
(\Sti)_{aa}
\frac{\log(s\eta^{-1})}{n} }
\right\}.
\end{equation}
Since $\Sti\hmt$ is a multivariate normal with  mean $\Sti\mt$ and
variance $\frac{n}{n_1n_2}\Sti$, we have $\PP[\Ecal_3(\eta)] \geq
1-\eta$.

Furthermore, define
 the event $\Ecal_4(\eta)$ as 
\begin{equation}
\label{eq:event_e4}
  \begin{aligned}
\Ecal_4(\eta) = \bigg\{
|(\hmt-\mt)'\Sti\sgn(\betab_{T})|  
\leq 
\sqrt{32
\Lambda_{\min}^{-1}(\St)
 \frac{ s \log(\eta^{-1})}{n}}
\bigg\}.
  \end{aligned}
\end{equation}
Since $\hmt'\Sti\sbt \sim \Ncal\left(\mt'\Sti\sbt, 
\frac{n}{n_1n_2}\sbt'\Sti\sbt\right)$, we have
$\PP[\Ecal_4(\eta)] \geq 1-\eta$.

The next result gives a deviation of $\hmt'\hSti\sbt$ from
$\mt'\Sti\sbt$.

\begin{lemma}
\label{lem:l1_norm}
The following inequality
\begin{equation}
\label{eq:lem:l1_norm:statement}
  \begin{aligned}
    &\abr{\hmt'\hSti\sbt - \mt'\Sti\sbt} \\
    & \qquad\qquad\leq C
    \rbr{
      \sqrt{\Lambda_{\min}^{-1}(\St)s\rbr{1\vee\snbt}}
      \vee
      \norm{\beta_T}_1
    }\sqrt{\frac{\log\log(n)}{n}}
  \end{aligned}
\end{equation}  
holds with probability at least $1-\Ocal(\log^{-1}(n))$.  
\end{lemma}

\begin{proof}
Using the triangle inequality
\begin{equation}
  \begin{aligned}
\label{eq:lem:l1_norm:proof:0}
&\abr{\hmt'\hSti\sbt - \mt'\Sti\sbt} \\
&\qquad\leq
    \abr{\hmt'\hSti\sbt - \hmt'\Sti\sbt}  \\
&\qquad\qquad    +
    \abr{\hmt'\Sti\sbt - \mt'\Sti\sbt}.
      \end{aligned}
\end{equation}
For the first term, we write
  \begin{equation}
\label{eq:lem:l1_norm:proof:1}
  \begin{aligned}
    &\abr{\hmt'\hSti\sbt - \hmt'\Sti\sbt} \\
    & \leq \sbt'\hSti\sbt \abr{
       \frac{\hmt'\hSti\sbt}{\sbt'\hSti\sbt}
       -
       \frac{\hmt'\Sti\sbt}{\sbt'\Sti\sbt}      
      } \\
    & \quad + 
    \sbt'\hSti\sbt
    \abr{ \frac{\sbt'\Sti\sbt}{\sbt'\hSti\sbt} - 1 }
    \abr{ \frac{\hmt'\Sti\sbt}{\sbt'\Sti\sbt} }.
  \end{aligned}
\end{equation}  
Let
\begin{equation*}
  \Gb =
  \left(
    \begin{array}{cc}
      \hmt'\Sti\hmt & \hmt'\Sti\sbt \\
      \sbt'\Sti\hmt & \sbt\Sti\sbt \\
    \end{array}
  \right),
\end{equation*}
and
\begin{equation*}
  \hat\Gb =
  \left(
    \begin{array}{cc}
      \hmt'\hSti\hmt & \hmt'\hSti\sbt \\
      \sbt'\hSti\hmt & \sbt\hSti\sbt \\
    \end{array}
  \right).
\end{equation*}
Using Theorem 3 of
\cite{Bodnar08properties}, 
we compute the density of $\hat z_a =
\hat{\Gb}_{12}\hat{\Gb}_{22}^{-1}$ conditional on $\hmt$
and obtain that
\begin{equation*}
    \sqrt{\frac{n-s}{q}}\left(
    \frac{\hmt'\hSti\sbt}{\sbt'\hSti\sbt}
    -
    \frac{\hmt'\Sti\sbt}{\sbt'\Sti\sbt}
    \right)\mid\hmt
    \sim
    t_{n-s},
\end{equation*}
where 
\begin{equation*}
q = 
\frac{\hmt'\left(\Sti-\frac{\Sti\sbt\sbt'\Sti}{\sbt'\Sti\sbt}\right)\hmt}
{\sbt'\Sti\sbt}
\leq \frac{\hmt'\Sti\hmt}{\sbt'\Sti\sbt}.
\end{equation*}
Lemma~\ref{lem:t-distribution} gives
\begin{equation*}
  \begin{aligned}
    &\abr{
    \frac{\hmt'\hSti\sbt}{\sbt'\hSti\sbt}
    -
    \frac{\hmt'\Sti\sbt}{\sbt'\Sti\sbt}
    } \\
    &\qquad\qquad\qquad
    \leq
    C \sqrt{
      \frac{\hmt'\Sti\hmt}{\sbt'\Sti\sbt}
      \frac{\log\log(n)}{n}
    },
    \end{aligned}
\end{equation*}
with probability at least $1-\log^{-1}(n)$. Combining with
Lemma~\ref{lem:ratio_chi}, Lemma~\ref{lem:hmt_sti_hmt}, and
\eqref{eq:event_e4}, we obtain an upper bound on the RHS of
\eqref{eq:lem:l1_norm:proof:1} as 
\begin{equation}
\label{eq:lem:l1_norm:proof:bound_t1}
  \begin{aligned}
    &\abr{\hmt'\hSti\sbt - \hmt'\Sti\sbt} \\
    & \qquad\qquad\leq C
    \rbr{
      \sqrt{\Lambda_{\min}^{-1}(\St)s\snbt}
      \vee
      \norm{\beta_T}_1
    }\sqrt{\frac{\log\log(n)}{n}}
  \end{aligned}
\end{equation}  
with probability at least $1-\Ocal(\log^{-1}(n))$.

The second term in \eqref{eq:lem:l1_norm:proof:0} can be bounded
using \eqref{eq:event_e4} with $\eta = \log^{-1}(n)$. Therefore,
combining with \eqref{eq:lem:l1_norm:proof:bound_t1}, we obtain 
\begin{equation}
\label{eq:lem:l1_norm:proof:bound_t0}
  \begin{aligned}
    &\abr{\hmt'\hSti\sbt - \mt'\Sti\sbt} \\
    & \qquad\qquad\leq C
    \rbr{
      \sqrt{\Lambda_{\min}^{-1}(\St)s\rbr{1\vee\snbt}}
      \vee
      \norm{\beta_T}_1
    }\sqrt{\frac{\log\log(n)}{n}}
  \end{aligned}
\end{equation}  
with probability at least $1-\Ocal(\log^{-1}(n))$,
as desired.

\end{proof}

\begin{lemma}
  \label{lem:ratio_chi}
  There exist constants $C_1,C_2,C_3$, and $C_4$ such that each of the
  following inequalities hold with probability at least
  $1-\log^{-1}(n)$ : 
\begin{align}
\label{eq:lem:ratio_chi:1}
    \eb_a'\hSti\eb_a  \leq 
    C_1 \eb_a'\Sti\eb_a\left(1 + 
      \Ocal\left(\sqrt{\frac{\log(s\log(n))}{n}}\right)
    \right), \quad \forall a \in T\\
\label{eq:lem:ratio_chi:2}
    \left| 
    \frac{\eb_a'\Sti\eb_a}{\eb_a'\hSti\eb_a}
    -1
    \right|
    \leq C_2 \sqrt{\frac{\log(s\log(n))}{n}},
    \quad \forall a \in T \\
\label{eq:lem:ratio_chi:3}
\abr{
\frac{\sbt'\Sti\sbt}{\sbt'\hSti\sbt} - 1}
\leq C_3\sqrt{\frac{\log\log(n)}{n}}, \quad\text{and}\\
\label{eq:lem:ratio_chi:4}
\sbt'\hSti\sbt \leq C_4\sbt'\Sti\sbt
\left(1 + 
      \Ocal\left(\sqrt{\frac{\log\log(n)}{n}}\right)
\right).
\end{align}

\end{lemma}
\begin{proof}
Theorem 3.2.12. in \cite{muirhead1982aspects} states that
  \begin{equation*}
    (n-2)\frac{\eb_a'\Sti\eb_a}{\eb_a'\hSti\eb_a} \sim
    \chi^2_{n-s-1}.
  \end{equation*}
Using Equation~\eqref{eq:chi-central-upper-johnstone},
\begin{equation*}
    \left| 
    \frac{n-2}{n-s-1}\frac{\eb_a'\Sti\eb_a}{\eb_a'\hSti\eb_a}
    -1
    \right|
    \leq \sqrt{\frac{16\log(2s\log(n))}{3(n-s-1)}}
\end{equation*}
  with probability $1-(2s\log(n))^{-1}$. Rearranging terms in the
  display above, we have that
  \begin{equation*}
    \eb_a'\hSti\eb_a \leq 
    C \eb_a'\Sti\eb_a\left(1 + 
      \Ocal\left(\sqrt{\frac{\log(s\log(n))}{n}}\right)
    \right)
  \end{equation*}
  and
  \begin{equation*}
    \left| 
    \frac{\eb_a'\Sti\eb_a}{\eb_a'\hSti\eb_a}
    -1
    \right|
    \leq C \sqrt{\frac{\log(s\log(n))}{n}}.
  \end{equation*}
A union bound gives \eqref{eq:lem:ratio_chi:1} and \eqref{eq:lem:ratio_chi:2}.

Equations \eqref{eq:lem:ratio_chi:3} and \eqref{eq:lem:ratio_chi:4}
are shown similarly.
\end{proof}

\begin{lemma}
  \label{lem:deviation:beta_t}
  There exist constants $C_1,C_2 > 0$ such that the following
  inequality
\begin{equation}
  \label{eq:lem:deviation:beta_t}
  \begin{aligned}
    \forall a \in T\ :\ 
    \abr{\eb_a'\hSti\hmt - \eb_a'\Sti\mt}  
    & \leq C_1\sqrt{
      \rbr{\Sti}_{aa}\rbr{1\vee\snbt}
      \frac{\log(s\log(n))}{n}
    } \\
    & \quad + C_2 
    \abr{\eb_a'\Sti\mt}
    \sqrt{
      \frac{\log(s\log(n))}{n}
    }
  \end{aligned}
\end{equation}
holds with probability at least $1-\Ocal(\log^{-1}(n))$.
\end{lemma}

\begin{proof}
Using the triangle inequality, we have
\begin{equation}
  \begin{aligned}
\label{eq:lem:deviation:beta_t:proof:0}
\abr{\eb_a'\hSti\hmt - \eb_a'\Sti\mt} &\leq
    \abr{\eb_a'\hSti\hmt - \eb_a'\Sti\hmt}  +
    \abr{\eb_a'\Sti\hmt - \eb_a'\Sti\mt}.
      \end{aligned}
\end{equation}
For the first term, we write
  \begin{equation}
\label{eq:lem:deviation:beta_t:proof:1}
  \begin{aligned}
    \abr{\eb_a'\hSti\hmt - \eb_a'\Sti\hmt}  
    & \leq \eb_a'\hSti\eb_a \left|     
       \frac{\eb_a'\hSti\hmt}{\eb_a'\hSti\eb_a}
       -
       \frac{\eb_a'\Sti\hmt}{\eb_a'\Sti\eb_a}
       \right| \\
    & \quad + 
    \eb_a'\hSti\eb_a
    \left| \frac{\eb_a'\Sti\eb_a}{\eb_a'\hSti\eb_a} - 1 \right|
    \left| \frac{\eb_a'\Sti\hmt}{\eb_a'\Sti\eb_a} \right|.
  \end{aligned}
\end{equation}  
As in the proof of Lemma~\ref{eq:lem:l1_norm:statement}, we can show that
\begin{equation*}
    \sqrt{\frac{n-s}{q_a}}\left(
    \frac{\eb_a'\hSti\hmt}{\eb_a'\hSti\eb_a}
    -
    \frac{\eb_a'\Sti\hmt}{\eb_a'\Sti\eb_a}
    \right)\mid\hmt
    \sim
    t_{n-s},
\end{equation*}
where 
\begin{equation*}
q_a = 
\frac{\hmt'\left(\Sti-\frac{\Sti\eb_a\eb_a'\Sti}{\eb_a'\Sti\eb_a}\right)\hmt}
{\eb_a'\Sti\eb_a}
\leq
\frac{\hmt'\Sti\hmt}{\eb_a'\Sti\eb_a}.
\end{equation*}
Lemma~\ref{lem:t-distribution} and an application of union bound gives
\[
\abr{
    \frac{\eb_a'\hSti\hmt}{\eb_a'\hSti\eb_a}
    -
    \frac{\eb_a'\Sti\hmt}{\eb_a'\Sti\eb_a}
}
\leq C\sqrt{
\frac{\hmt'\Sti\hmt}{\eb_a'\Sti\eb_a}
\frac{\log(s\log(n))}{n}
},
\qquad
\forall a \in T,
\]
with probability at least $1-\Ocal(\log^{-1}(n))$. Combining
Lemma~\ref{lem:hmt_sti_hmt}, Lemma~\ref{lem:ratio_chi} and
Equation~\eqref{eq:event_e3} with $\eta = \log^{-1}(n)$, we can bound the
right hand side of \eqref{eq:lem:deviation:beta_t:proof:1} as 
\begin{equation}
  \label{eq:lem:deviation:beta_t:proof:bound:t1}
  \begin{aligned}
    \abr{\eb_a'\hSti\hmt - \eb_a'\Sti\hmt}  
    & \leq C_1\sqrt{
      \rbr{\Sti}_{aa}\snbt
      \frac{\log(s\log(n))}{n}
    } \\
    & \quad + C_2 
    \abr{\eb_a'\Sti\mt}
    \sqrt{
      \frac{\log(s\log(n))}{n}
    }
  \end{aligned}
\end{equation}
with probability at least $1-\Ocal(\log^{-1}(n))$.

The second term in \eqref{eq:lem:deviation:beta_t:proof:0} is handled
by \eqref{eq:event_e3} with $\eta=\log^{-1}(n)$. Combining with
\eqref{eq:lem:deviation:beta_t:proof:bound:t1}, we obtain
\begin{equation}
  \label{eq:lem:deviation:beta_t:proof:bound:t0}
  \begin{aligned}
    \abr{\eb_a'\hSti\hmt - \eb_a'\Sti\mt}  
    & \leq C_1\sqrt{
      \rbr{\Sti}_{aa}\rbr{1\vee\snbt}
      \frac{\log(s\log(n))}{n}
    } \\
    & \quad + C_2 
    \abr{\eb_a'\Sti\mt}
    \sqrt{
      \frac{\log(s\log(n))}{n}
    }
  \end{aligned}
\end{equation}
with probability at least $1-\Ocal(\log^{-1}(n))$.
This completes the proof.
\end{proof}

\begin{lemma}
\label{lem:hsti_sbt}
The probability of the event
\begin{equation*}
\label{eq:hsti_sbt}
\begin{aligned}
\bigcap_{a\in[s]}\Bigg\{
&|\eb_a'(\hSti - \Sti)\sbt|
\\
& \ 
\leq C \left(
  \sqrt{(\Sti)_{aa}\Lambda_{\min}^{-1}(\St)s}
  \vee
  |\eb_a'\Sti\sbt|
\right)
    \sqrt{\frac{\log(s\log(n))}{n}}
\Bigg\}
\end{aligned}
\end{equation*}
is at least  $1-2\log^{-1}(n)$ for $n$ sufficiently large.
\end{lemma}

\begin{proof}
  Write
  \begin{equation}
\label{eq:lem:hsti_sbt:proof:1}
  \begin{aligned}
    |\eb_a'&\hSti\sbt - \eb_a'\Sti\sbt| \\
    & \leq \eb_a'\hSti\eb_a \left|     
       \frac{\eb_a'\hSti\sbt}{\eb_a'\hSti\eb_a}
       -
       \frac{\eb_a'\Sti\sbt}{\eb_a'\Sti\eb_a}
       \right| \\
    & \quad + 
    \eb_a'\hSti\eb_a
    \left| \frac{\eb_a'\Sti\eb_a}{\eb_a'\hSti\eb_a} - 1 \right|
    \left| \frac{\eb_a'\Sti\sbt}{\eb_a'\Sti\eb_a} \right|.
  \end{aligned}
  \end{equation}  
As in the proof of Lemma~\ref{eq:lem:l1_norm:statement}, we can show that
\begin{equation*}
    \sqrt{\frac{n-s}{q_a}}\left(
    \frac{\eb_a'\hSti\sbt}{\eb_a'\hSti\eb_a}
    -
    \frac{\eb_a'\Sti\sbt}{\eb_a'\Sti\eb_a}
    \right)
    \sim
    t_{n-s},
\end{equation*}
where 
\begin{equation*}
q_a = 
\frac{\sbt'\left(\Sti-\frac{\Sti\eb_a\eb_a'\Sti}{\eb_a'\Sti\eb_a}\right)\sbt}
{\eb_a'\Sti\eb_a}.
\end{equation*}
Therefore,
\begin{equation}
\label{eq:lem:hsti_sbt:proof:4}
    \Bigg|
    \frac{\eb_a'\hSti\sbt}{\eb_a'\hSti\eb_a}
    -
    \frac{\eb_a'\Sti\sbt}{\eb_a'\Sti\eb_a}
    \Bigg|    
     \leq C
    \sqrt{
      q_a
      \frac{\log(s\log(n))}{n}}
\end{equation}
with probability $1 - (s\log(n))^{-1}$.  Combining
Lemma~\ref{lem:ratio_chi} and \eqref{eq:lem:hsti_sbt:proof:4}, we can
bound the right hand side of Equation~\eqref{eq:lem:hsti_sbt:proof:1} as
\begin{equation*}
  \begin{aligned}
    |\eb_a'&\hSti\sbt - \eb_a'\Sti\sbt| \\
    & \leq C \rbr{
            (\Sti)_{aa}\sqrt{q_a} \vee
            |\eb_a'\Sti\sbt|
         }
    \sqrt{\frac{\log(s\log(n))}{n}}\\
    & \leq C \rbr{
      \sqrt{(\Sti)_{aa}^{-1}\Lambda_{\min}^{-1}(\St)s}
      \vee
            |\eb_a'\Sti\sbt|
         }
    \sqrt{\frac{\log(s\log(n))}{n}}
  \end{aligned}
\end{equation*}
where the second inequality follows from 
\[
q_a \leq
(\Sti)_{aa}^{-1}\sbt'\Sti\sbt \leq
(\Sti)_{aa}^{-1}\Lambda_{\min}^{-1}(\St)s.
\] An application of the union bound gives the desired result.
\end{proof}

\section{Tail Bounds For Certain Random Variables}

In this section, we collect useful results on tail bounds of various
random quantities used throughout the paper. We start by stating a
lower and upper bound on the survival function of the standard normal
random variable. Let $Z \sim \Ncal(0,1)$ be a standard normal random
variable. Then for $t > 0$
\begin{equation}
  \label{eq:tail-bound-normal}
  \frac{1}{\sqrt{2\pi}} \frac{t}{t^2+1}\exp(-t^2/2)
  \leq \PP(Z > t) 
  \leq \frac{1}{\sqrt{2\pi}}\frac{1}{t} \exp(-t^2/2).
\end{equation}

Next, we collect results concerning tail bounds for central $\chi^2$
random variables.
\begin{lemma}[\cite{Laurent00adaptive}] Let $X \sim \chi^2_d$. For all
  $x \geq 0$,
\begin{align}
  \label{eq:chi-central-upper}
  \PP[X - d \geq 2 \sqrt{dx} + 2x] & \leq \exp(-x)\\
  \label{eq:chi-central-lower}
  \PP[X - d \leq -2 \sqrt{dx}] & \leq \exp(-x).
\end{align}
\end{lemma}

\begin{lemma}[\cite{johnstone2009consistency}]
  Let $X \sim \chi^2_d$, then
  \begin{equation}
    \label{eq:chi-central-upper-johnstone}
    \mathbb{P}[|d^{-1}X -1| \geq x] \leq \exp(-\frac{3}{16}dx^2), \quad x
    \in [0, \frac{1}{2}).  
  \end{equation}
\end{lemma}

The following result provides a tail bound for non-central $\chi^2$
random variable with non-centrality parameter $\nu$.

\begin{lemma}[\cite{birge2001alternative}] Let $X \sim
  \chi^2_d(\nu)$, then for all $x > 0$
\begin{align}
  \label{eq:chi-noncentral-upper}
  \PP[X \geq (d+\nu) + 2 \sqrt{(d+2\nu)x} + 2x] & \leq \exp(-x)\\
  \label{eq:chi-noncentral-lower}
  \PP[X \leq (d+\nu)-2 \sqrt{(d+2\nu)x}] & \leq \exp(-x).
\end{align}
\end{lemma}

The following Lemma gives a tail bound for a $t$-distributed random variable.

\begin{lemma}
  \label{lem:t-distribution}
  Let $X$ be a random variable distributed as 
  \[
     X \sim \sigma d^{-1/2} t_d,
  \]
  where $t_d$ denotes a $t$-distribution with $d$ degrees of
  freedom. Then
  \[
  |X| \leq C\sqrt{\sigma^2 d^{-1}\log(4\eta^{-1})}
  \]
  with probability at least $1-\eta$.
  
\end{lemma}

\begin{proof}
  Let $Y \sim \Ncal(0,1)$ and $Z \sim \chi^2_d$ be two independent
  random variables. Then $X$ is equal in distribution to
  \[
    \frac{\sigma d^{-1/2} Y}{\sqrt{d^{-1}Z}}.
  \]
  Using \eqref{eq:tail-bound-normal}, 
  \[
  |\sigma d^{-1/2} Y| \leq \sigma d^{-1/2}\sqrt{\log(4\eta^{-1})}
  \]
  with probability at least
  $1-\eta/2$. \eqref{eq:chi-central-upper-johnstone} gives 
  \[
  d^{-1}X \geq 1 - \sqrt{\frac{16}{3d}\log(2\eta^{-1})}
  \]
  with probability at least $1-\eta/2$.
  Therefore, for sufficiently large $d$,
  \[
  |X| \leq C\sqrt{\sigma^2 d^{-1}\log(4\eta^{-1})}.
  \]
  
\end{proof}

\bibliography{local}

\end{document}